\documentclass[12pt]{article}
\usepackage{times}
\usepackage{graphicx}
\usepackage{color}
\usepackage{url}
\usepackage{multirow}
\usepackage[authoryear]{natbib}
\usepackage{rotating}
\usepackage{bbm}
\usepackage{latexsym}
\usepackage{jmlr2e}

\newtheorem{OP}{{\sf OP}}

\usepackage{caption}
 \usepackage{url}  
\usepackage{graphicx}  
\usepackage{amsfonts}
\usepackage{braket}
\usepackage{boxedminipage}
\usepackage{epsf}
\usepackage{bm}
\usepackage{amsmath}
\usepackage{xspace}
\usepackage{wrapfig}
\usepackage{algorithm,algpseudocode}
\usepackage{color}
\usepackage{framed}

\newtheorem{defi}{Definition}
\newtheorem{theo}{Theorem}

\newtheorem{coro}[theo]{Corollary}
\newtheorem{lemm}{Lemma}

\newcommand{\modiff}[1]{\textcolor{black}{#1}} 
\newcommand{\rmmodif}[1]{} 

\algnewcommand{\Inputs}[1]{%
  \State \textbf{Inputs:}
  \Statex \hspace*{\algorithmicindent}\parbox[t]{.8\linewidth}{\raggedright #1}
}
\algnewcommand{\Initialize}[1]{%
  \State \textbf{Initialize:}
  \Statex \hspace*{\algorithmicindent}\parbox[t]{.8\linewidth}{\raggedright #1}
}
\def\indot<#1>{\langle #1 \rangle}

\begin{document}

\title{Theory and Algorithms for Shapelet-based Multiple-Instance Learning \thanks{This is the preprint version of a paper published in Neural Computation.}}





\author{
 Daiki Suehiro\\
 \email{suehiro@ait.kyushu-u.ac.jp} \\
 \addr 
 Kyushu University\\
and AIP, RIKEN\\
\AND
 Kohei Hatano\\
 \email{hatano@inf.kyushu-u.ac.jp}\\
 \addr 
 Kyushu University\\
and AIP, RIKEN\\
\AND
 Eiji Takimoto\\
 \email{eiji@inf.kyushu-u.ac.jp}\\
 \addr 
 Kyushu University\\
\AND
 Shuji Yamamoto\\
 \email{yamashu@math.keio.ac.jp}\\
 \addr 
 Keio University\\
and AIP, RIKEN\\
\AND
 Kenichi Bannai\\
 \email{bannai@math.keio.ac.jp}\\
 \addr 
 Keio University\\
and AIP, RIKEN\\
\AND
 Akiko Takeda\\
 \email{takeda@mist.i.u-tokyo.ac.jp}\\
 \addr 
 The University of Tokyo\\
and AIP, RIKEN\\
}


 \maketitle




\begin{abstract}
We propose a new formulation of  Multiple-Instance Learning (MIL),
in which a unit of data consists of a set of instances
  called a bag. The goal is to find a good classifier of bags based on
  the similarity with a ``shapelet'' (or pattern), where the similarity
  of a bag with a shapelet is the maximum similarity of
  instances in the bag.
In previous work, some of the training instances are chosen
as shapelets with no theoretical justification.
In our formulation, we use all possible, and thus infinitely many
shapelets, resulting in a richer class of classifiers.
We show that the formulation is tractable, 
that is, it can be reduced through Linear Programming Boosting (LPBoost) to
Difference of Convex (DC)
programs of finite (actually polynomial) size.
\modiff{
Our theoretical result also gives justification to the heuristics
of some of the previous work.
The time complexity of the proposed algorithm highly
depends on the size of the set of all instances in the training sample.
To apply to the data containing a large number of instances,
we also propose a heuristic option of the algorithm
without the loss of the theoretical guarantee.
Our empirical study demonstrates that 
our algorithm uniformly works for Shapelet Learning tasks on
time-series classification 
and various MIL tasks
with comparable accuracy to the existing methods.
Moreover, we show that the proposed heuristics allow us to achieve
the result with reasonable computational time.
}
\end{abstract}

\section{Introduction}
Multiple-Instance Learning (MIL) is a fundamental framework of
supervised learning with a wide range of applications such as
prediction of molecular activity, and image classification.
MIL has been extensively studied both in theoretical and practical
aspects~\citep{Gartner02multi-instancekernels,NIPS2002misvm,Sabato:2012:MLA,pmlr-v28-zhang13a,Doran:2014,CARBONNEAU2018329},
since the notion of MIL was first proposed by~\citet{Dietterich:1997}.

A standard MIL setting is described as follows:
A learner receives sets $B_1, B_2, \ldots, B_m$
called bags; each contains multiple instances. 
In the training phase, each bag is labeled
but instances are not labeled individually.
The goal of the learner is to obtain a hypothesis 
that predicts the labels of 
unseen bags correctly\footnote{Although there are settings where
instance label prediction is also considered, 
we focus only on bag-label prediction in this paper.}.
One of the most common hypotheses used in practice has the following form:
\begin{align}
\label{align:single-shape}
h_{\bf u}(B) = \max_{x \in B} \left\langle {\bf u}, \Phi(x)\right\rangle,
\end{align}
where $\Phi$ is a feature map and 
${\bf u}$ is a feature vector which we call a \emph{shapelet}.
In many applications, ${\bf u}$ is
interpreted as a particular ``pattern'' in the feature space
and the inner product
as the similarity of $\Phi(x)$ from ${\bf u}$.
Note that we use the term ``shapelets'' by following the terminology
of Shapelet Learning (SL), which is a framework for time-series classification, 
although it is often called ``concepts'' in the literature of MIL.
Intuitively, this hypothesis evaluates a given bag by the
maximum similarity between the instances in the bag and the shapelet ${\bf u}$.
Multiple-Instance Support Vector Machine (MI-SVM) proposed by~\citet{NIPS2002misvm} 
is a widely used algorithm that
employs this hypothesis class and 
learns ${\bf u}$.
It is well-known that MIL algorithms using this hypothesis class
perform empirically better in various multiple-instance datasets.
Moreover, a generalization error bound of the hypothesis class 
is given by~\citet{Sabato:2012:MLA}.

However, in some domains such as 
image recognition and document classification,
it is said that the hypothesis class (\ref{align:single-shape}) is
not effective \citep[see, e.g.,][]{MI1normSVM}.
To employ MIL on such domains more effectively,
\modiff{
\citet{MI1normSVM} extend a hypothesis to
a convex combination of $h_{\bf u}$ :
\begin{align}
\label{align:our-hypo}
 	 g(B) = \sum_{{\bf u} \in U} w_{\bf u} \max_{x \in B} \left\langle {\bf u},
  \Phi(x)\right\rangle, 
\end{align}
for some set $U$ of shapelets. In particular, Chen et al. consider
$U_\mathrm{train}=\{\Phi(z) \mid z \in \bigcup_{i=1}^{m} B_i\}$, which is
constructed from all instances in the training sample. 
}
The authors demonstrate that
this hypothesis with the Gaussian kernel performs well
in image recognition.
\modiff{
The generalization bound provided by \citet{Sabato:2012:MLA} is
applicable to a hypothesis class of the form~(\ref{align:our-hypo}) for the set
$U$ of infinitely many shapelets ${\bf u}$ with bounded norm.
Therefore, the generalization bound also holds for $U_\mathrm{train}$.
However, it has never been theoretically discussed why
such a fixed set $U_\mathrm{train}$ using training instances effectively works in MIL tasks.
}
\subsection{Our Contributions}
\label{subsec:contributions}
 In this paper, 
we propose an MIL formulation with the hypothesis class
(\ref{align:our-hypo}) for sets $U$ of infinitely many shapelets.

The proposed learning framework is 
theoretically motivated and practically effective. 
We show the generalization error bound 
based on the Rademacher complexity~\citep{Bartlett:2003:RGC} 
and large margin theory.
The result indicates that we can achieve a small generalization
error by keeping 
a 
large margin for large training
sample. 

The learning framework
can be applied to various kinds of data and tasks
because of our unified formulation.
The existing shapelet-based methods are formulated for their
target domains. 
More precisely, the existing shapelet-based methods
are formulated using a fixed similarity measure (or distance),
and the generalization ability is
shown empirically in their target domains.
For example, \citet{MI1normSVM} and \citet{sangnier16} 
calculated the feature vectors based on 
the similarity between every instance using the Gaussian kernel.
In time-series domain, shapelet-based methods
\citep{Ye:2009:TSS:1557019.1557122,KeoghR13,Hills:2014:CTS:2597434.2597448} 
usually use Euclidean distance as a similarity measure (or distance).
By contrast, our framework employs a kernel function as a
similarity measure.
Therefore, 
our learning framework can be uniformly applied
if we can set a kernel function as 
a 
similarity measure
according to a target learning task.
For example, the Gaussian kernel (behaves like the Euclidean distance) 
and Dynamic Time Warping (DTW) kernel~\citep{Shimodaira:2001}. 
Our framework can be also applied
to non-real-valued sequence data 
(e.g., text, and a discrete signal) using a string kernel.
Moreover, our generalization performance
is guaranteed theoretically.
The experimental results demonstrate that 
the provided approach uniformly works for SL and MIL tasks
without introducing 
domain-specific  
parameters and heuristics, and
compares with the state-of-the-art shapelet-based methods.

We show that the formulation is tractable.
The algorithm is 
based on Linear Programming Boosting
\citep[LPBoost,][]{demiriz-etal:ml02} that solves the soft margin optimization problem
via a column generation approach.
Although
the weak learning problem in the boosting becomes 
an optimization 
problem over an infinite-dimensional space,
we can show that an analog of the representer theorem holds on it 
and allows us to reduce it to a non-convex optimization problem
(difference of convex program)
over a finite-dimensional space.
While it is difficult to solve the sub-problems exactly because of
non-convexity, it is possible to find
good approximate solutions 
with reasonable time
in many practical cases \citep[see, e.g.,][]{LeThi2018}.

Remarkably,
our theoretical result gives justification to the heuristics of
choosing the shapelets in the training instances.
\modiff{
Our representer theorem indicates that 
at $t$-th iteration of boosting,
the optimal solution ${\bf u}_t$ (i.e., shapelet) of the weak learning problem 
can be written as a linear combination of 
the feature maps of training instances, that is, ${\bf u}_t=\sum_{z \in \bigcup_{i=1}^m B_i} \alpha_{t,z} \Phi(z)$.
Thus, we obtain a final classifier of the following form
\[
g(B) = \sum_{t=1}^Tw_t\max_{x \in B}  \langle {\bf u}_t, \Phi(x) \rangle = 
\sum_{t=1}^Tw_t\max_{x \in B} \sum_{z \in \bigcup_{i=1}^m B_i} \alpha_{t,z} \langle \Phi(z),
\Phi(x) \rangle.
\]
Note that the hypothesis class used in the 
standard approach~\citep{MI1normSVM,sangnier16}
corresponds to the special case 
where ${\bf u}_t \in U_\mathrm{train}=\{\Phi(z) \mid z \in \bigcup_{i=1}^{m} B_i\}$. 
This observation would suggest that the standard approach
of using $U_\mathrm{train}$ is reasonable.
}

\subsection{Comparison to Related Work for MIL}
\label{subsec:related}
There are many MIL algorithms 
with hypothesis classes which are different from
(\ref{align:single-shape}) or (\ref{align:our-hypo}). 
\citep[e.g.,][]{mi_adaboost,Gartner02multi-instancekernels,NIPS2003DPBOOST,NIPS2005_2926,MI1normSVM}.
For example, these algorithms adopted diverse approaches for the bag-labeling hypothesis
from shapelet-based hypothesis classes 
(e.g., \citet{NIPS2005_2926} used a Noisy-OR based hypothesis and 
\citet{Gartner02multi-instancekernels} proposed a new kernel
called a set kernel).
Shapelet-based hypothesis classes have a practical advantage 
of being applicable to SL in 
the 
time-series domain 
(see next subsection).

\citet{Sabato:2012:MLA} proved generalization bounds of hypotheses
classes for MIL including those of (\ref{align:single-shape}) and 
(\ref{align:our-hypo}) with infinitely large
sets $U$.
\modiff{
The generalization bound we provided in this paper is incomparable to the bound provided by Sabato and Tishby. 
When some data-dependent parameter is regarded as a constant,
our bound is slightly better in terms of the sample size $m$ by the factor of $O(\log m)$. 
}
They also proved the PAC-learnability of the class
(\ref{align:single-shape}) using the boosting approach under some
technical assumptions.
Their boosting approach is different from our work in that they assume
that labels are consistent with some hypothesis of the
form~(\ref{align:single-shape}), while we consider arbitrary distributions
over bags and labels.
\subsection{Connection between MIL and Shapelet Learning for Time Series
  Classification}
\label{subsec:connection}
Here we mention briefly that MIL with
type (\ref{align:our-hypo}) hypotheses
is closely related to SL, a framework
for time-series classification that has been extensively
studied~\citep{Ye:2009:TSS:1557019.1557122,KeoghR13,Hills:2014:CTS:2597434.2597448,Grabocka:2014:LTS:2623330.2623613}
in parallel to MIL.
SL is a notion of learning with a feature extraction\rmmodif{(remove a
  comma here)} method, defined by
a finite set $M \subseteq \mathbb{R}^\ell$ of
real-valued ``short'' sequences called shapelets.
A similarity measure is given by (not necessarily a Mercer kernel)
$K: \mathbb{R}^\ell \times \mathbb{R}^\ell \to \mathbb{R}$
in the following way.
A time series ${\boldsymbol{\tau}} = (\tau[1], \dots, \tau[L]) \in \mathbb{R}^L$
can be identified with a bag
$B_{\boldsymbol{\tau}} = \{(\tau[j], \ldots, \tau[j+\ell-1])
\mid 1 \leq j \leq L-\ell+1\}$ consisting of all subsequences
of ${\boldsymbol{\tau}}$ of length $\ell$.
The feature of ${\boldsymbol{\tau}}$ is a vector
	$\left(\max_{{\bf x} \in B_{\boldsymbol{\tau}}} K({\bf z}, {\bf x}) \right)_{{\bf z} \in M}$
of a fixed dimension $|M|$ regardless of the length
$L$ of the time series ${\boldsymbol{\tau}}$.
When we employ a linear classifier on top of the features, we obtain
a hypothesis in the form:
\begin{align}
\label{align:multi-shape-ts}
	g({\boldsymbol{\tau}}) = \sum_{{\bf z} \in M} w_{\bf z} \max_{{\bf x} \in B_\tau} K({\bf z},{\bf x}),
\end{align}
which is essentially the same form as~(\ref{align:our-hypo}),
except that
finding 
good shapelets $M$ is a part of the learning task,
as well as to 
find 
a good weight vector ${\bf w}$.
This task is one of the most successful 
approaches  
for SL
\citep{Hills:2014:CTS:2597434.2597448,Grabocka:2014:LTS:2623330.2623613,GrabockaWS15,renard:hal-01217435,HouKZ16}, where
a typical choice of $K$ is $K({\bf z},{\bf x}) = -\|{\bf z} - {\bf x}\|_2$.
However, almost all existing methods heuristically choose
shapelets $M$ and with no theoretical guarantee on how good
the choice of $M$ is.

Note also that in the SL framework, each ${\bf z} \in M$ is called a
shapelet, while in this paper, we assume that $K$ is a kernel
$K(z,x) = \langle \Phi(z), \Phi(x) \rangle$ and
any ${\bf u}$ (not necessarily $\Phi(z)$ for some $z$)
in the Hilbert space is called a shapelet.

\citet{sangnier16} proposed an MIL-based anomaly
detection algorithm for time series data. 
They showed an algorithm based on LPBoost and the generalization error
bound based on the Rademacher complexity~\citep{Bartlett:2003:RGC}.
Their hypothesis class is same as~\citep{MI1normSVM}.
However, they did not 
analyze 
the theoretical justification to use
finite set $U$ made from training instances 
(the authors mentioned as
future work).
By contrast, we consider a hypothesis class based on 
infinitely many shapelets, and
our representer theorem guarantees that our learning problem
over the infinitely large set is still tractable.
As a result, our study justifies the previous 
heuristics of their approach.

There is another work which treats shapelets
not appearing in the training set.
Learning Time-Series Shapelets (LTS)
algorithm~\citep{Grabocka:2014:LTS:2623330.2623613} 
tries to solve a non-convex optimization problem of learning 
effective shapelets in an infinitely large domain. 
However, there is no theoretical guarantee of its generalization
error.
In fact, our generalization error bound applies 
to 
their hypothesis class. \par
For SL tasks, many researchers focus on improving efficiency~\citep{KeoghR13,renard:hal-01217435,GrabockaWS15,WistubaGS15,HouKZ16,Karlsson:2016}.
However, these methods are specialized 
in the 
time-series domain, 
and the generalization performance has never been theoretically discussed.

Curiously, despite MIL and SL share similar motivations and
hypotheses, the relationship between MIL and SL has not yet been 
pointed out.
From the shapelet-perspective in MIL, the hypothesis (\ref{align:single-shape})
is regarded as a ``single shapelet''-based hypothesis, and
the hypothesis (\ref{align:our-hypo}) is regarded as 
a 
``multiple-shapelets''-based hypothesis.
In this study, we refer to a linear combination of maximum similarities based on
shapelets such as (\ref{align:our-hypo}) and
(\ref{align:multi-shape-ts})
as {\it shapelet-based classifiers}.

\section{Preliminaries}
\label{sec:prelim}
Let ${\mathcal{X}}$ be an instance space.
A bag $B$ is a finite set of instances chosen from ${\mathcal{X}}$. 
The learner receives a sequence of labeled bags
$S = ((B_1, y_1), \ldots, (B_m, y_m)) \in
(2^{{\mathcal{X}}} \times \{-1, 1\})^m$ called a sample, where 
each labeled bag is independently drawn according to some unknown
distribution $D$ over $2^{{\mathcal{X}}} \times \{-1, 1\}$.
Let ${P_S}$ denote the
set of all instances that appear in the sample $S$.
That is, ${P_S} = \bigcup_{i=1}^m B_i$.
Let $K$ be a kernel over ${\mathcal{X}}$, which is used to measure the
similarity between instances, and let
$\Phi: {\mathcal{X}} \to {\mathbb{H}}$ denote a feature map associated with
the kernel $K$ for a Hilbert space ${\mathbb{H}}$,
that is,
$K(z, z')=\langle \Phi(z), \Phi(z')\rangle$
for instances $z, z' \in {\mathcal{X}}$, where
$\langle \cdot,\cdot \rangle$ denotes the inner product over ${\mathbb{H}}$.
The norm induced by the inner product is denoted by
$\|\cdot\|_{\mathbb{H}}$ defined as $\|{\bf u}\|_{{\mathbb{H}}} =
\sqrt{\langle {\bf u}, {\bf u} \rangle}$ for ${\bf u} \in {\mathbb{H}}$. 

For each ${\bf u} \in {\mathbb{H}}$ which we call a shapelet,
we define 
a {\it shapelet-based classifier}
denoted by $h_{{\bf u}}$, 
as the function that maps a given bag
$B$ to the maximum of
the similarity scores between shapelet ${\bf u}$ and $\Phi(x)$
over all instances $x$ in $B$.
More specifically,
\[
	h_{{\bf u}}(B) = \max_{x \in B}
		\left\langle{\bf u}, \Phi(x) \right\rangle.
\]
For a set $U \subseteq {\mathbb{H}}$, we define the class of shapelet-based classifiers
as 
\[
	H_U = \left\{ h_{{\bf u}} \mid {\bf u} \in U \right\}
\]
and let ${\mathrm{conv}}(H_U)$ denote the set of convex combinations of
shapelet-based classifiers in $H_U$. More precisely,
\begin{align} \nonumber
\label{align:final_hypo}
	 {\mathrm{conv}}(H_U) = 
       &\left\{
		\int_{{\bf u} \in U} w_{\bf u} h_{\bf u} d{\bf u} \mid
		\text{$w_{\bf u}$ is a density over $U$}
          \right\} \\ \nonumber
       =&\left\{
		\sum_{{\bf u} \in U'} w_{\bf u} h_{\bf u} \mid
		\forall {\bf u} \in U', w_{\bf u} \geq 0, \right. \\
		&\left. \sum_{{\bf u} \in U'} w_{\bf u} = 1,
		\text{$U' \subseteq U$ is a finite support} 
 	\right\}.
\end{align}
The goal of the learner is to find a hypothesis
$g \in {\mathrm{conv}}(H_U)$, so that its generalization error
${\mathcal{E}}_D(g) = \Pr_{(B,y) \sim D}[{\mathrm{sign}}(g(B)) \neq y]$
is small.
Note that since the final hypothesis ${\mathrm{sign}} \circ g$
is invariant to any scaling of $g$, we assume
without loss of generality that
\[
	U = \{{\bf u} \in {\mathbb{H}} \mid \|{\bf u}\|_{\mathbb{H}} \leq 1\}.
\]
Let
${\mathcal{E}}_{\rho}(g)$ denote the \textit{empirical margin loss} of $g$
over $S$,
that is, ${\mathcal{E}}_{\rho}(g) = |\{i \mid y_i g(B_i) < \rho\}|/m$.

\section{Optimization Problem Formulation}
\label{sec:opt}
In this paper, 
we formulate the problem as
soft margin maximization with $1$-norm regularization,
which ensures a generalization bound
for the final hypothesis \citep[see, e.g.,][]{demiriz-etal:ml02}.
Specifically, the problem is formulated as a
linear programming problem (over infinitely many variables) as
follows:
\begin{align}\label{align:LPBoostPrimal}
\max\limits_{\rho,w,{\bf \xi}} \; &
 \rho  - \frac{1}{\nu m}\sum_{i=1}^m \xi_{i}
\\ \nonumber
 \text{sub.to} \; 
& \int_{{\bf u} \in U}y_i w_{{\bf u}}h_{\bf u}(B_i) 
d{\bf u} 
                       \geq \rho -\xi_{i} \wedge \xi_{i} \geq0, ~ i \in [m],
\\ \nonumber
&\int_{{\bf u} \in U} w_{{\bf u}}d{\bf u}  = 1,
 w_{\bf u}  \geq 0,~ \rho \in \mathbb{R},
\end{align}
where $\nu \in [0,1]$ is a parameter.
To avoid the integral over the Hilbert space,
it is convenient to consider the dual form:
\begin{align}\label{align:LPBoostDual}
\min\limits_{\gamma,{\bf d}} \; & \gamma
\\ \nonumber
\text{sub.to} \; 
& \sum_{i=1}^my_id_i h_{\bf u}(B_i) \leq \gamma,
\; {\bf u} \in U,
\\ \nonumber
& 0 \leq d_i \leq 1/(\nu m),\; i \in [m],
\\ \nonumber
& \sum_{i=1}^m d_{i}  = 1, ~ \gamma \in \mathbb{R}.
\end{align}
The dual problem is categorized as a semi-infinite program because
it contains infinitely many constraints.
Note that the duality gap is zero because the problem
(\ref{align:LPBoostDual})
is linear and the optimum is finite \citep[see Theorem 2.2 of][]{Shapiro09}.
We employ column generation to solve the dual problem:
solve (\ref{align:LPBoostDual}) for a finite subset $U' \subseteq U$,
find ${\bf u}$ to which the corresponding constraint is maximally violated
by the current solution (\emph{column generation part}),
and repeat the procedure with $U' = U' \cup \{{\bf u}\}$ 
until a certain stopping criterion is met.
In particular, we use
LPBoost~\citep{demiriz-etal:ml02}, a well-known and practically
fast algorithm of column generation.
Since the solution ${\bf w}$ is expected to be sparse
due to the 1-norm regularization,
the number of iterations is expected to be small.

Following the boosting terminology, we refer to the
column generation part as weak learning.
In our case, weak learning is formulated following
the optimization problem:
\begin{align}\label{align:WeakLearn_u}
\max_{{\bf u} \in {\mathbb{H}}} \;
\sum_{i=1}^my_id_i
\max_{x \in B_i}  \left\langle {\bf u},  \Phi\left(x \right)
                       \right\rangle \; \text{sub.to} \; \|{\bf u}\|_{\mathbb{H}}^2 \leq 1.
\end{align}
Thus, we need to design a weak learner for solving
(\ref{align:WeakLearn_u}) for a given sample weighted by ${\bf d}$.
However, it seems to be impossible to solve it directly because
we only have access to $U$ through the associated kernel.
Fortunately, we prove a version of representer theorem given below,
which makes (\ref{align:WeakLearn_u}) tractable.
\begin{theo}[Representer Theorem]\label{theo:represent}
The solution ${\bf u}^*$ of (\ref{align:WeakLearn_u}) can be
written as
${\bf u}^* = \sum_{z \in {P_S}} \alpha_z \Phi(z)$
for some real numbers $\alpha_z$.
\end{theo}
Our theorem can be derived from a non-trivial application 
of the standard representer theorem \citep[see, e.g.,][]{Mohri.et.al_FML}.
Intuitively, we prove the theorem by decomposing the optimization problem
(\ref{align:WeakLearn_u})
into a number of sub-problems, so that the standard representer theorem
can be applied to each of the sub-problems.
The detail of the proof is given in 
Appendix~\ref{sec:proof1}.
\par
\modiff{
This result gives justification 
to the simple heuristics in the standard approach: choosing the shapelets
based on the training instances.
More precisely, the hypothesis class used in the 
standard approach~\citep{MI1normSVM,sangnier16}
corresponds to the special case 
where ${\bf u} \in U_\mathrm{train}=\{\Phi(z) \mid z \in P_S\}$. 
Thus, our representer theorem would suggest that the standard approach
of using $U_\mathrm{train}$ is reasonable.
}

Theorem~\ref{theo:represent} says that the weak learning problem
can be rewritten in the following tractable form:
\begin{OP}\label{align:WeakLearn}{\bf Weak Learning Problem}
\begin{align*}
\min_{{\boldsymbol{\alpha}}} \;& 
- \sum_{i=1}^m{d}_i y_i
\max_{x \in B_i} \sum_{z \in {P_S}} \alpha_{z} K\left(
                        z, x \right)
\\ \nonumber
\mathrm{sub.to} \; & 
\sum_{z \in {P_S}}\sum_{v \in {P_S}}\alpha_{z}\alpha_{v}K \left(z, v
\right)  \leq 1.
\end{align*}
\end{OP}



%
Unlike the primal solution ${\bf w}$, the dual solution
${\boldsymbol{\alpha}}$ is not expected to be sparse.
To obtain a more interpretable hypothesis,
we propose another formulation of weak learning where
1-norm regularization is imposed on ${\boldsymbol{\alpha}}$, so that
a sparse solution of ${\boldsymbol{\alpha}}$ will be obtained.
In other words, instead of $U$, we consider the feasible set
$\hat{U} = \left\{\sum_{z \in {P_S}}\alpha_{z} \Phi(z):
\|{\boldsymbol{\alpha}}\|_1 \leq 1\right\}$,
where $\|{\boldsymbol{\alpha}}\|_1$ is the 1-norm of ${\boldsymbol{\alpha}}$.
\begin{OP}\label{align:WeakLearnSP} {\bf Sparse Weak Learning Problem}
\begin{align*}
\min_{{\boldsymbol{\alpha}}} \; & 
- \sum_{i=1}^m{d}_i y_i
\max_{x \in B_i} \sum_{z \in {P_S}} \alpha_{z} K\left(
                        z, x\right)
\\
\mathrm{sub.to} \;\; & \|{\boldsymbol{\alpha}}\|_1  \leq 1
\end{align*}
\end{OP}
Note that when running LPBoost with a weak learner for
\textsf{OP}~\ref{align:WeakLearnSP}, we obtain a final hypothesis
that has the same form of generalization
bound as the one stated in Theorem~\ref{theo:main}, which
is of a final hypothesis obtained when used with a weak learner
for \textsf{OP}~\ref{align:WeakLearn}.
To see this, consider a feasible space
$\hat U_\Lambda = \left\{\sum_{z \in {P_S}}\alpha_{z} \Phi(z):
\|{\boldsymbol{\alpha}}\|_1 \leq \Lambda \right\}$
for a sufficiently small $\Lambda > 0$, so that
$\hat U_\Lambda \subseteq U$. Then since
$H_{\hat U_\Lambda} \subseteq H_U$, a generalization bound
for $H_U$ also applies to $H_{\hat U_\Lambda}$.
On the other hand, since the final hypothesis ${\mathrm{sign}} \circ g$
for $g \in {\mathrm{conv}}(H_{\hat U_\Lambda})$ is invariant to the
scaling factor $\Lambda$, the generalization ability is independent
of $\Lambda$.
\section{Algorithms}
\label{sec:algo}
In this section, we present the 
pseudo-code 
of LPBoost
in Algorithm~\ref{alg:LPBoost} for completeness.
Moreover,
we describe our algorithms for the weak learners.
For simplicity, we denote by ${\mathbf{k}}_x \in \mathbb{R}^{P_S}$
a vector given by $k_{x,z} = K(z,x)$ for every $z \in P_S$.
The objective function of \textsf{OP}~\ref{align:WeakLearn}
(and \textsf{OP}~\ref{align:WeakLearnSP})
is rewritten as
\[
	\sum_{i:y_i = -1} d_i \max_{x \in B_i} {\mathbf{k}}_x^T {\boldsymbol{\alpha}}
	- \sum_{i:y_i = 1} d_i \max_{x \in B_i} {\mathbf{k}}_x^T {\boldsymbol{\alpha}},
\]
which can be seen as a difference $F - G$ of two convex functions 
$F$ and $G$ of ${\boldsymbol{\alpha}}$.
Therefore, the weak learning problems are DC programs and thus
we can use DC algorithm~\citep{Tao1988,Yu:2009:LSS}
to find an $\epsilon$-approximation of a local optimum.
We employ a standard DC algorithm. That is, for each
iteration $t$, we linearize the concave term $G$
with $\nabla_{{\boldsymbol{\alpha}}} G({\boldsymbol{\alpha}}_t)^T {\boldsymbol{\alpha}}$ at
the current solution ${\boldsymbol{\alpha}}_t$, which is 
$\sum_{i:y_i=1} d_i {\mathbf{k}}_{x_i^*}^T {\boldsymbol{\alpha}}$ with
$x_i^* = \arg\max_{x \in B_i} {\mathbf{k}}_x^T {\boldsymbol{\alpha}}$ in our case,
and then update the solution to ${\boldsymbol{\alpha}}_{t+1}$ by solving the resultant
convex optimization problem $\textsf{OP}'_t$.

In addition, the problems $\textsf{OP}'_t$
for \textsf{OP}~\ref{align:WeakLearn} and \textsf{OP}~\ref{align:WeakLearnSP}
are reformulated as a second-order cone programming (SOCP) problem
and an LP problem, respectively, and thus both
problems can be efficiently solved.
To this end, we introduce new variables
$\lambda_i$ for all negative bags $B_i$ with $y_i = -1$
which represent the factors
$\max_{x \in B_i} {\mathbf{k}}_x^T {\boldsymbol{\alpha}}$.
Then we obtain the equivalent problem to $\textsf{OP}'_t$
for \textsf{OP}~\ref{align:WeakLearn}
as follows:
\begin{align}\label{align:WeakLearnSub}
\min_{{\boldsymbol{\alpha}}, {\boldsymbol{\lambda}}} \; & 
\sum_{i:y_i= -1}d_i \lambda_i
- \sum_{i:y_i = 1} d_i \max_{x \in B_i} {\mathbf{k}}_{x}^T {\boldsymbol{\alpha}} 
\\ \nonumber
\text{sub.to} \; &  {\mathbf{k}}_x^T {\boldsymbol{\alpha}}
                      \leq \lambda_i 
                      ~(\forall i:y_i=-1, \forall x \in B_i), \\ \nonumber
&\sum_{z \in {P_S}}\sum_{v \in {P_S}}\alpha_{z}\alpha_{v}K \left(z, v \right)  \leq 1.
\end{align}
It is well known that this is an SOCP problem.
Moreover, it is clear that $\textsf{OP}'_t$ for
\textsf{OP}~\ref{align:WeakLearnSP} can be formulated as an LP problem.
We describe the algorithm for \textsf{OP}~\ref{align:WeakLearn}
in Algorithm~\ref{alg:WeakLearn}.

One may concern that a kernel matrix may become large
when a sample consists 
a 
large amount of bags and instances.
However, note that the kernel matrix of $K(z, x)$ which is used in
Algorithm~\ref{alg:WeakLearn}
needs to be computed only once 
at the beginning of Algorithm~\ref{alg:LPBoost}, not at every iteration.

\modiff{
As a result, our learning algorithm outputs a classifier
\[
g(B) = \mathrm{sign}\left(\sum_{t=1}^T{w_t}\max_{x \in B}\sum_{z \in P_S} \alpha_{t, z} K (z, x) \right)
\]
where $w_t$ and ${\boldsymbol{\alpha}}_t$ are obtained in training phase.
Therefore, the computational cost for predicting the label of $B$ is
$O(T|P_S||B|)$ in the worst case when all elements of $\alpha_{t,z}$ are non-zero. 
However, when we employ
our sparse formulation {\sf OP} 2 which allows us to find a sparse $\alpha$,
the computational cost is expected to be much smaller than the worst case.
}
\begin{algorithm}[h!]
\caption{LPBoost using WeakLearn}
\label{alg:LPBoost}
\begin{algorithmic}[0]
\Inputs{$S$, kernel $K$, $\nu \in (0, 1]$, $\epsilon > 0$}
\Initialize{${\bf d}_0 \leftarrow (\frac{1}{m}, \ldots, \frac{1}{m}), \gamma=0$}
\For{ $t=1,\dots $}
 \State $h_t\leftarrow$  Run {\bf WeakLearn}($S$, $K$, ${\bf d}_{t-1}$,
   $\epsilon$)
 \If {$\sum_{i=1}^my_id_i h_t(B_i) \leq \gamma$}
 \State $t = t-1$, break
 \EndIf
\State
\begin{minipage}{5cm}
\begin{align}
\nonumber
& (\gamma, {\bf d}_t) \leftarrow \arg\min\limits_{\gamma,{\bf d}} \;  \gamma 
\\ \nonumber
&\text{sub.to}~~ 
 \sum_{i=1}^my_id_i h_j(B_i) \leq \gamma
~~(j = 1, \ldots, t), ~~
\\ \nonumber
&0 \leq d_i \leq 1/\nu m ~~(i \in [m]),  \sum_{i=1}^m d_i  = 1, ~ \gamma \in \mathbb{R}.
\end{align}
\end{minipage}
\EndFor
\State ${\bf w} \leftarrow$ Lagrangian multipliers of the last solution 
\State $g \leftarrow \sum_{j=1}^t w_j h_j$ \\
\Return {${\mathrm{sign}}(g)$}
\end{algorithmic}
\end{algorithm}

\begin{algorithm}[h!]
\caption{WeakLearn using the DC Algorithm}
\label{alg:WeakLearn}
\begin{algorithmic}[0]
\Inputs{$S$, $K$, 
    ${\bf d}$, $\epsilon$ (convergence parameter)}
 \Initialize{${\boldsymbol{\alpha}}_0 \in \mathbb{R}^{|{P_S}|}$, $f_0 \leftarrow \infty$}
 \For{$t=1,\dots $}
\For{$\forall k:y_k=+1$}
\State 
$\displaystyle x^*_k \leftarrow \arg\max_{x \in B_k} \sum_{z \in {P_S}}
   d_k\alpha_{t,z} K\left(z, x\right)$ 
\EndFor 
\State
\begin{minipage}{8cm}
\begin{align}\label{align:subprob}
 f \leftarrow \min_{{\boldsymbol{\alpha}}, \boldsymbol{\lambda}} 
\;& 
- \sum_{k:y_k=+1} {d}_k
\sum_{z \in {P_S}} \alpha_{z} K\left(z, x_k^*\right)
+ \sum_{r:y_r=-1}{d}_r \lambda_r 
\\ \nonumber
\text{sub.to} \;&  \sum_{z \in {P_S}} \alpha_{z} K\left(z, x  \right) 
                      \leq \lambda_r 
~(\forall r:y_r=-1, \forall x \in B_r),\\ \nonumber
&\sum_{z \in {P_S}}\sum_{v \in {P_S}} \alpha_{z}\alpha_{v}K
  \left(z, v \right)  \leq 1.
\end{align}
\end{minipage}
\State ${\boldsymbol{\alpha}}_t \leftarrow {\boldsymbol{\alpha}}$, $f_t \leftarrow f$
\If{$f_{t-1} - f_{t} \leq \epsilon$} 
\State break
\EndIf 
\EndFor \\
\Return{$h(B) =  \max_{x \in
    B}\sum_{z \in {P_S}}\alpha_{t,z}K(z, x)$}
\end{algorithmic}
\end{algorithm}

\section{Generalization Bound of the Hypothesis Class}
\label{sec:theorem1}
In this section, we provide a generalization bound of
hypothesis classes ${\mathrm{conv}}(H_U)$ for various $U$ and $K$.

Let $\Phi(P_S)=\{\Phi(z) \mid z \in P_S\}$.
Let ${\Phi_{\mathrm{diff}}(P_S)}=\{\Phi(z) - \Phi(z') \mid
z,z'\in P_S, z \neq z'\}$.
By viewing each instance ${{\bf v}} \in {\Phi_{\mathrm{diff}}(P_S)}$ as a hyperplane
$\{{\bf u} \mid \langle {{\bf v}}, {\bf u} \rangle = 0\}$,
we can naturally define a partition
of the Hilbert space ${\mathbb{H}}$
by the set of all hyperplanes ${{\bf v}} \in {\Phi_{\mathrm{diff}}(P_S)}$.
Let ${\mathcal{I}}$ be the set of all cells of the partition, that is,
${\mathcal{I}}=\{ I \mid
I=\cap_{{{\bf v}} \in V}
\{{\bf u} \mid  \langle {{\bf v}}, {\bf u} \rangle > 0\}, I \neq \emptyset,  V \subseteq \Phi_{\mathrm{diff}}(P_S), {\bf v} \in V\Leftrightarrow - {\bf v} \notin V \mathrm{~for~all~} {\bf v }\in \Phi_{\mathrm{diff}}(P_S)
\}$.
Each cell $I \in {\mathcal{I}}$ is a polyhedron which is defined by
a minimal set $V_I \subseteq {\Phi_{\mathrm{diff}}(P_S)}$ that
satisfies $I = \bigcap_{{{\bf v}} \in V_I}
\{{\bf u} \mid \langle {\bf u}, {{\bf v}} \rangle > 0\}$.
Let
\[
	\mu^* = \min_{I \in {\mathcal{I}}}\max_{{\bf u}\in I \cap U}\min_{{{\bf v}} \in V_I}
 |\langle {\bf u}, {{\bf v}} \rangle|. 
\]
Let $d^*_{\Phi,S}$ be the VC dimension of the set of linear classifiers
over the finite set ${\Phi_{\mathrm{diff}}(P_S)}$, given by
$F_U=\{ f: {{\bf v}} \mapsto {\mathrm{sign}}(\langle {\bf u}, {{\bf v}}\rangle) \mid {\bf u} \in U\}$.

Then we have the following generalization bound on the hypothesis class
of (\ref{align:our-hypo}).

\begin{theo}
\label{theo:main}
Let $\Phi: {\mathcal{X}} \rightarrow {\mathbb{H}}$.
Suppose that for any $z \in {\mathcal{X}}$, $\|\Phi(z)\|_{\mathbb{H}} \leq R$.
Then, for any $\rho>0$, 
with high probability 
the following holds
for any $g \in {\mathrm{conv}}(H_U)$ with
$U \subseteq \{{\bf u} \in {\mathbb{H}} \mid \|{\bf u}\|_{\mathbb{H}} \leq 1\}$:
\begin{align}
 {\mathcal{E}}_D(g) \leq& {\mathcal{E}}_{\rho}(g)
 +O\left(
 \frac{R \sqrt{d_{\Phi, S}^* \log |{P_S}|}}{\rho\sqrt{m}}  
\right),
\end{align}
where 
 (i) for any $\Phi$,
 $d^*_{\Phi,S}=O((R/\mu^*)^2)$,
(ii) if ${\mathcal{X}} \subseteq \mathbb{R}^\ell$ 
and $\Phi$ is the identity mapping (i.e., the associated kernel
 is the linear kernel), or
(iii) if ${\mathcal{X}} \subseteq \mathbb{R}^\ell$ 
and $\Phi$ satisfies the condition that
$\left\langle\Phi(z), \Phi(x) \right\rangle$ is monotone
decreasing with respect to $\|z-x\|_2$  (e.g.,
the mapping defined by the Gaussian kernel) and
$U=\{\Phi(z) \mid z \in \mathbb{R}^\ell, \|\Phi(z)\|_{\mathbb{H}} \leq 1\}$, 
 then $d^*_{\Phi, S}=O(\min((R/\mu^*)^2, \ell))$.
\end{theo}
We show the proof in Appendix~\ref{proof2}.

\paragraph{Comparison with the existing bounds}
A similar generalization bound 
can be derived from 
a known bound of the Rademacher complexity of $H_U$ \citep[Theorem 20 of][]{Sabato:2012:MLA}
and a generalization bound of ${\mathrm{conv}(H)}$ for any hypothesis class $H$
\citep[see Corollary 6.1 of][]{Mohri.et.al_FML}:
\[
{\mathcal{E}}_D(g) \leq {\mathcal{E}}_{\rho}(g) + O\left(\frac{{\log
      \left(\sum_{i=1}^m|B_i| \right)\log(m)}}{\rho\sqrt{m}} \right).
\]
Note that \citet{Sabato:2012:MLA} fixed $R=1$. 
For simplicity, we omit some constants of \citep[Theorem 20 of][]{Sabato:2012:MLA}.
Note that $|{P_S}| \leq \sum_{i=1}^m|B_i|$ by definition.
The bound above is incomparable to
Theorem~\ref{theo:main} in general, as ours uses the parameter $d^*_{\Phi,S}$ and
the other has the extra $\sqrt{\log\left(\sum_{i=1}^m|B_i|
  \right)}\log(m)$ term.
However, our bound is better in terms of the sample size $m$ by the
factor of $O(\log m)$ when other parameters are regarded as constants.

\section{SL by MIL}
\subsection{Time-Series Classification with Shapelets}
In the following, we introduce a framework of time-series classification 
problem based on shapelets (i.e., SL problem). 
As mentioned in the previous section,
a time series ${\boldsymbol{\tau}} = (\tau[1], \dots, \tau[L]) \in \mathbb{R}^L$
can be identified with a bag
$B_{\boldsymbol{\tau}} = \{(\tau[j], \ldots, \tau[j+\ell-1])
\mid 1 \leq j \leq L-\ell+1\}$ that consists of all subsequences
of ${\boldsymbol{\tau}}$ of length $\ell$.
The learner receives a labeled sample 
$S = ((B_{{\boldsymbol{\tau}}_1}, y_1), \ldots, (B_{{\boldsymbol{\tau}}_m}, y_m)) \in (2^{\mathbb{R}{^\ell}}
\times \{-1, 1\})^m$, 
where each labeled bag (i.e. labeled time series) 
is independently drawn 
according to some unknown distribution $D$ over a finite
support of $2^{\mathbb{R}^{\ell}} \times \{-1, +1\}$.
The goal of the learner is to predict the labels of 
an unseen time series correctly.
In this way, the SL problem can be viewed as an MIL problem,
and thus we can apply our algorithms and theory.

Note that, for time-series classification, various similarity measures 
can be represented by a kernel.
For example, the Gaussian kernel (behaves like the Euclidean distance)
and Dynamic Time Warping (DTW) kernel.
Moreover, our framework can generally apply to non-real-valued sequence data 
(e.g., text, and a discrete signal) using a string kernel.

\subsection{Our Theory and Algorithms for SL}
\label{sec:theory_SL}
By Theorem~\ref{theo:main}, 
we can immediately obtain the generalization bound of 
our hypothesis class in SL as follows:
\begin{coro}
\label{coro:ts}
Consider time-series sample $S$ of size $m$ and length $L$.
For any fixed $\ell < L$, the following generalization error bound
holds for all $g \in {\mathrm{conv}}(H_U)$ in which the length of shapelet
is $\ell$:
\[
{\mathcal{E}}_D(g) \leq {\mathcal{E}}_{\rho}(g) +
O\left(
 \frac{R \sqrt{d_{\Phi, S}^* \log (m(L-\ell +1))}}{\rho\sqrt{m}} \right).
\]
\end{coro}
To the best of our knowledge, this is the first result on the 
generalization performance
of SL. \par
Theorem~\ref{theo:represent} gives justification 
to the heuristics which choose the shapelets
extracted from the instances appearing in the training sample
(i.e., the subsequences for SL tasks).
Moreover, several methods using 
a 
linear combination of shapelet-based
classifiers \citep[e.g.,][]{Hills:2014:CTS:2597434.2597448,Grabocka:2014:LTS:2623330.2623613},
are supported by Corollary~\ref{coro:ts}.

For time-series classification problem, shapelet-based classification
has a greater advantage of 
the interpretability or visibility
than other
time-series classification methods \citep[see, e.g.,][]{Ye:2009:TSS:1557019.1557122}.
Although we use a nonlinear kernel function,
we can observe important subsequences 
that contribute to effective shapelets by solving OP~\ref{align:WeakLearnSP} 
because of the sparsity (see also the experimental results).
Moreover, for unseen time-series data,
we can observe the type of subsequences that contribute
to the predicted class
by observing maximizer $x \in B$. 
\subsection{Learning Shapelets of Different Lengths}
\label{sec:general_shapelets}
For time-series classification, 
many existing methods take advantage of using shapelets of various lengths.
Below, we show that our formulation can be easily applied to the case.

A time series
${\boldsymbol{\tau}} = (\tau[1], \dots, \tau[L]) \in \mathbb{R}^L$
can be identified with a bag
$B_{\boldsymbol{\tau}} = \{(\tau[j], \ldots, \tau[j+\ell-1])
\mid 1 \leq j \leq L-\ell+1, \forall \ell \in Q \}$ that consists of all
length $\ell \in Q \subseteq \{1, \ldots, L\} $ of subsequences of ${\boldsymbol{\tau}}$.
That is, this is also a special case of MIL that 
a bag contains different dimensional instances.

There is a simple way to
apply our learning algorithm to this case.
We just employ some kernels $K(z, x)$ which supports
different dimensional instance pairs $z$ and $x$.
Fortunately, such kernels have been studied well in the time-series domain.
For example, DTW kernel and Global Alignment kernel~\citep{cuturi2011fast}
are well-known kernels which support time series of different lengths.
However, the size of the kernel matrix of $K(z, x)$ 
becomes 
$m(\sum_{\ell \in Q}(L-\ell +1))^2$.
In practice, it requires high memory cost for large time-series data.
Moreover, in general, the above kernel requires a higher computational cost
than standard kernels.

We introduce a practical way to learn shapelets of different lengths
based on heuristics.
In each weak learning problem, 
we decomposed the original weak learning problem over different dimensional data space
into the weak learning problems over each dimensional data space.
For example, we consider solving the following problem instead of 
the weak learning problem \textsf{OP}~\ref{align:WeakLearn}:
\begin{align*}
\min_\ell\min_{{\boldsymbol{\alpha}}} \;& 
- \sum_{i=1}^m{d}_i y_i
\max_{x \in B_{i}^\ell} \sum_{z \in {P_S}^\ell} \alpha_{z} K\left(
                        z, x \right),\\ 
\text{sub.to} \; & 
\sum_{z \in {P_S}^\ell}\sum_{v \in {P_S}^\ell}\alpha_{z}\alpha_{v}K \left(z, v \right)  \leq 1.
\end{align*}
where $B_{i}^\ell$ denotes the $\ell$ dimensional instances 
(i.e., length $\ell$ of subsequences) in $B_i$,
and ${P_S}^\ell$ denotes $\bigcup_{i=1}^m B_i^\ell$.
The total size of kernel matrices 
becomes 
$m\sum_{\ell \in Q}((L-\ell +1))^2$,
and thus this method does not require 
so large kernel matrix.
Moreover, in this way, we do not need to use a kernel which supports different dimensional
instances.
Note that, even using this heuristic, the obtained final hypothesis 
has theoretical generalization performance. 
This is because the hypothesis class still represented as the form of 
(\ref{align:final_hypo}).
In our experiment, we use the latter method by giving weight to memory efficiency.

\subsection{Heuristics for computational efficiency}
\label{sec:exp_app}
For the practical applications, we introduce some heuristics for improving efficiency
in our algorithm.
\paragraph{Reduction of ${P_S}$}
Especially for time-series data, 
the size $|{P_S}|$ often becomes large 
because $|{P_S}| = O(mL)$.
Therefore, constructing a kernel matrix of $|{P_S}| \times |{P_S}|$
has high computational costs for time-series data.
For example, when we 
consider subsequences as instances for time series classification, 
we have a large computational cost because of the number of subsequences 
of training data (e.g., approximately $10^6$ when sample size is $1000$ and
length of each time series is $1000$, 
which results in a similarity matrix of size $10^{12}$). 
However, in most cases, many subsequences in time series data 
are similar to each other.
Therefore, we only use representative instances ${\hat{P}_S}$
instead of the set of all instances ${P_S}$. 
In this paper, we use $k$-means clustering 
to reduce the size of $|{P_S}|$.
Note that our heuristic approach is still supported by our
theoretical generalization error bound.
This is because the hypothesis set $H_{U'}$ with
the reduced shapelets $U'$ is the subset of $H_U$,
and the Rademacher complexity of $H_{U'}$ is exactly
smaller than the Rademacher complexity of $H_U$. 
Thus, Theorem~\ref{theo:main} holds for the hypothesis 
class considering the set $H_U$ of all possible shapelets $U$, 
and thus Theorem~\ref{theo:main} also holds for
the hypothesis class using the set $H_{U'}$ 
of some reduced shapelets $U'$.
Although this approach may decrease the training classification accuracy in practice,
it drastically decreases the computational cost for a large dataset.

\paragraph{Initialization in weak learning problem}
DC program may slowly converge to local optimum 
depending on the initial solution.
In Algorithm~\ref{alg:WeakLearn}, we fix an initial
${\boldsymbol{\alpha}}_{0}$ as following:
More precisely, we initially solve 
\begin{align}
\label{align:shape_opt}
{\boldsymbol{\alpha}}_0 = \arg\max_{{\boldsymbol{\alpha}}}& \sum_{i=1}^md_iy_i\max_{x \in B_i} \sum_{z \in {P_S}}
   \alpha_{z} K\left(z, x \right), \\ \nonumber
   ~~\text{sub.to}&~~ {\boldsymbol{\alpha}} ~\text{is a one-hot vector.}
\end{align}
That is, we choose the most discriminative shapelet from ${P_S}$
as the initial point of ${\bf u}$ for given ${\bf d}$.
We expect that it will speed up the convergence 
of the loop of line 3, and the obtained classifier is better 
than the methods that choose effective 
shapelets from subsequences.

\section{Experiments}
\label{sec:experiments}
In this section, 
we show some experimental results
implying that our algorithm performs comparably
with the existing shapelet-based classifiers for
both SL and MIL tasks~\footnote{The code of our method is available in \url{https://github.com/suehiro93/MILIMS_NECO}}.

\subsection{Results for Time-Series Data}
\label{subsec:exp_ts}
We use binary labeled datasets\footnote{Note that our method is
applicable to multi-class classification tasks by easy expansion
(e.g., \cite{NIPS1999_1773}).}
available in UCR datasets~\citep{UCRArchive}, which are often used as
benchmark datasets for time-series classification methods.
We used a weak learning problem OP~\ref{align:WeakLearnSP}
because the interpretability of the obtained classifier is required 
in shapelet-based time-series classification.

We compare the following three shapelet-based approaches.
\begin{itemize}
\item Shapelet Transform (ST) provided by~\citet{Bagnall2017}
\item Learning Time-Series Shapelets (LTS) provided by~\citet{Grabocka:2014:LTS:2623330.2623613} 
\item Our algorithm using shapelets of different lengths (Ours)
\end{itemize}
We used the implementation of ST provided by~\citet{sktime},
and used the implementation of LTS provided by~\citet{tslearn}.
\modiff{
The classification rule of Shapelets Transform has the form:
\[
g(B)=f\left(\max_{x \in B}-\|z_1 - x\|, \ldots, \max_{x \in B}-\|z_k - x\| \right),
\]
where $f$ is a user-defined classification function (the implementation employs decision forest),
$z_1, \ldots, z_k \in P_S$ (in the time-series domain, this $z_j$ is called a shapelet).
The shapelets are chosen from training subsequences in some complicated way 
before learning $f$.
The classification rule of Learning Time-series Shapelets has the form:
\[
g(B)=\sum_{j=1}^k w_j\max_{x \in B}-\|z_j - x\|,
\]
where $w_j \in \mathbb{R}$ and $z_j \in \mathbb{R}^{\ell}$ are learned parameters,
the number of desired shapelets $k$ is a hyper-parameter.
}

Below we show the detail condition of the experiment.
For ST, 
we set the shapelet lengths $\{2, \dots, L/2\}$,
where $L$ is the length of each time series in the dataset.
ST also requires a parameter of time limit for searching shapelets,
and we set it as 5 hours for each dataset.
For LTS, we used the hyper-parameter sets (regularization 
parameter, 
number of shapelets, etc.)
that the 
authors 
recommended in their
website\footnote{\url{http://fs.ismll.de/publicspace/LearningShapelets/}},
and we found an 
optimal 
hyper-parameter by $3$-fold cross-validation for
each dataset.
For our algorithms, we implemented a weak learning algorithm
which supports shapelets of different lengths (see Section~\ref{sec:general_shapelets}).
In this experiment, we consider the case that 
each bag contains lengths 
$\{0.05,  0.1,  0.15,  \ldots, 0.5\}\times L$ of the subsequences. 
We used the Gaussian kernel $K(x, x') = \exp(-\frac{\|x - x'\|^2}{\ell\sigma^2})$,
chose $1/\sigma^2$ from $\{0.01, 0.05, \allowbreak 0.1, \ldots, 50 \}$.
We chose $\nu$ from $\{0.1, 0.2, 0.3, 0.4\}$.
We use $100$-means clustering with respect to each class to reduce $P_S$.
The parameters we should tune are only $\nu$ and $\sigma$. 
We tuned these parameters via a procedure we give 
in Appendix~\ref{sec:exp_app2}.
As an LP solver for 
WeakLearn and LPBoost we used the CPLEX software.
In addition to Ours, LTS 
employs 
$k$-means clustering
to set the 
initial 
shapelets in the optimization algorithm.
Therefore, we show the average accuracies
for LTS and Ours considering the randomness of $k$-means clustering.

The classification accuracy results
are shown in Table~\ref{tab:acc1}.
We can see that our algorithms achieve comparable
performance with ST and LTS.
\modiff{
We conducted the Wilcoxon signed-rank test between Ours and 
the others.
The $p$-value of Wilcoxon signed-rank test for
Ours and ST is 0.1247.
The $p$-value of Wilcoxon signed-rank test for
Ours and LTS is 0.6219.
The $p$-values are higher than 0.05,
and thus we cannot rejcect that there is no significant difference between the medians of the accuracies.
}
We can say that our MIL algorithm works well for 
time-series classification tasks without using 
domain-specific knowledge.
\par
\modiff{
We would like to compare the computation time of these methods.
We selected the datasets that these three methods have achieved similar performance.
The experiments are performed on
Intel Xeon Gold 6154, 36 core CPU, 192GB memory.
Table~\ref{tab:ts_traintime}
shows the comparison of the running time of the training.
Note that again, for ST, we set the limitation of the running time as $5$ hours for finding good shapelets.
This running time limitation is a hyper-parameter of the code and it is difficult to be estimated before experiments.
LTS efficiently worked compared with ST and Ours. 
However, it seems that LTS achieved lower performance than ST and Ours on accuracy.
Table~\ref{tab:ts_testtime} shows the testing time of the methods.
LTS also efficiently worked, simply because LTS finds effective shapelets of a fixed number (hyper-parameter).
ST and Ours may find a large number of shapelets and this increases the computation time of prediction.
For Wafer dataset, ST and Ours required large computation time compared with LTS.}
\par
\modiff{
We can not fairly compare the efficiency of these methods because 
the implementation environments (e.g., programming languages) are different.
However, we can say that the proposed method totally achieved high classification accuracy
with reasonable running time for training and prediction.}

\renewcommand{\arraystretch}{0.7}
\begin{table*}
\centering
\caption{Classification accuracies for time-series datasets. \label{tab:acc1}}
\begin{tabular}{|c ||c|c|c|} \hline
Dataset & ST & LTS & Ours \\ \hline
BeetleFly & 0.8 & 0.765 & {\bf 0.835} \\
BirdChicken & 0.9 & 0.93 & {\bf 0.935} \\
Coffee & 0.964 & {\bf 1} & 0.964 \\
Computers & {\bf 0.704} & 0.619 & 0.623 \\
DistalPhalanxOutlineCorrect & 0.757 & 0.714 & {\bf 0.802} \\
Earthquakes & 0.741 & {\bf 0.748} & 0.728 \\
ECG200 & 0.85 & 0.835 & {\bf 0.872} \\
ECGFiveDays & 0.999 & 0.961 & {\bf 1} \\
FordA & 0.856 & {\bf 0.914} & 0.89 \\
FordB & 0.74 & {\bf 0.9} & 0.786 \\
GunPoint & {\bf 0.987} & 0.971 & {\bf 0.987} \\
Ham & 0.762 & {\bf 0.782} & 0.698 \\
HandOutlines & {\bf 0.919} & 0.892 & 0.87 \\
Herring & 0.594 & {\bf 0.652} & 0.588 \\
ItalyPowerDemand & 0.947 & {\bf 0.951} & 0.943 \\
Lightning2 & 0.639 & 0.695 & {\bf 0.779} \\
MiddlePhalanxOutlineCorrect & {\bf 0.794} & 0.579 & 0.632 \\
MoteStrain & {\bf 0.927} & 0.849 & 0.845 \\
PhalangesOutlinesCorrect & 0.773 & 0.633 & {\bf 0.792} \\
ProximalPhalanxOutlineCorrect & {\bf 0.869} & 0.742 & 0.844 \\
ShapeletSim & 0.994 & 0.989 & {\bf 1} \\
SonyAIBORobotSurface1 & {\bf 0.932} & 0.903 & 0.841 \\
SonyAIBORobotSurface2 & {\bf 0.922} & 0.895 & 0.887 \\
Strawberry & 0.941 & 0.844 & {\bf 0.947} \\
ToeSegmentation1 & {\bf 0.956} & 0.947 & 0.906 \\
ToeSegmentation2 & 0.792 & {\bf 0.886} & 0.823 \\
TwoLeadECG &{\bf  0.995} & 0.981 & 0.949 \\
Wafer & {\bf 1} & 0.993 & 0.991 \\
Wine & {\bf 0.741} & 0.487 & 0.72 \\
WormsTwoClass & {\bf 0.831} & 0.752 & 0.608 \\
Yoga & {\bf 0.847} & 0.69 & 0.804 \\ \hline
\end{tabular}
\end{table*}
\renewcommand{\arraystretch}{1.0}

\begin{table*}[h!]
\centering
\caption{Training time (sec.) for several time series datasets. \label{tab:ts_traintime}}
\begin{tabular}{|c |c|c||c|c|c| } \hline
dataset & \#train&length&
ST & LTS & 
Ours 
 \\ \hline
Earthquakes & $322$& $512$ & 
$18889.8$ & $250.5$ & $1339.2$
\\ 
GunPont &  $50$ & $150$ 
& $18016.2$  & $22.3$ & $36.9$
\\
ItalyPowerDemand & $67$& $24$ & 
$18000.8$ & $11.5$ & $8.6$  
\\ 
ShapeletSim & $20$& $180$ 
& $18011.6$ &$30.4$& $32.8$
\\ 
Wafer & $1000$ & $152$ 
& $18900.8$ & $91.5$ & $431.7$
\\ 
\hline
\end{tabular}
\end{table*}

\begin{table*}[h!]
\centering
\caption{Testing time (sec.) for several time series datasets. \label{tab:ts_testtime}}
\begin{tabular}{|c |c|c||c|c|c| } \hline
dataset & \#test&length&
ST & LTS & 
Ours 
 \\ \hline
Earthquakes & $139$& $512$ & 
$389.7$ & $2.75$ & $11.55$
\\ 
GunPont &  $150$ & $150$ 
& $48.0$  & $1.1$ & $3.9$
\\
ItalyPowerDemand & $1029$& $24$ & 
$3.3$ & $0.5$ & $10.7$ 
\\ 
ShapeletSim & $180$& $180$ 
& $104.0$ &$1.8$& $1.1$
\\ 
Wafer & $6164$ & $152$ 
& $5688.2$ & $4.3$ & $173.1$
\\ 
\hline
\end{tabular}
\end{table*}

\paragraph{Interpretability of our method}
We would like to show the interpretability of our method.
We use CBF dataset which contains three classes 
(cylinder, bell, and funnel) of time series.
The reason is that, it is known that
the discriminative patterns are clear, and thus 
we can easily ascertain if the obtained hypothesis
can capture the effective shapelets.
For simplicity, we obtain 
a 
binary classification model
for each class preparing one-vs-others training set.
We used Ours with fixed shapelet length $\ell=25$.
As following, we introduce two types of visualization
approach to interpret a learned model.

One is the visualization of the 
characteristic subsequences of an input time series.
When we predict the label of the time series $B$, 
we calculate a maximizer $x^*$ in $B$ for each $h_{\bf u}$,
that is, $x^*=\arg\max_{x \in B} \langle {\bf u}, \Phi(x)\rangle$.
For image recognition tasks,
the maximizers are commonly used to observe
the sub-images that characterize the 
class of the input image \citep[e.g.,][]{MI1normSVM}.
In time-series classification tasks, 
the maximizers also can be used
to observe some characteristic subsequences.
Fig.~\ref{fig:cbf_subseq} is an example
of 
a 
visualization of maximizers.
Each value in the
legend indicates $w_{\bf u} \max_{x \in B}\langle{\bf u}, \Phi(x)\rangle$.
That is, subsequences with positive 
values contribute to 
the positive class and subsequences 
with negative 
values contribute to
the negative class.
Such visualization provides
the subsequences that
characterize the class of the input time series.
For cylinder class, although both positive and negative patterns
match almost the same subsequence,
the positive pattern is stronger than negative, and thus 
the hypothesis can correctly discriminate the time series.
For bell and funnel class, we can observe that 
the highlighted subsequences clearly indicate 
the discriminative patterns.

The other is the visualization of a final hypothesis
$g(B) = \sum_{j=1}^tw_jh_j(B)$, where $h_j(B)=\max_{x \in B}\sum_{z_j \in
 {\hat{P}_S}} \alpha_{j, z_j}K(z_j, x)$ (${\hat{P}_S}$ is the set
of representative subsequences
obtained by $k$-means clustering).
Fig.~\ref{fig:cbf_shapelet}
is an example of 
the 
visualization
of a final hypothesis obtained by our algorithm.
The colored lines are all the 
$z_j$s in $g$ where both $w_j$ and $\alpha_{j, z_j}$ were non-zero.
Each legend value shows the multiplication of 
$w_j$ and $\alpha_{j, z_j}$ corresponding to $z_j$.
That is, positive values of the colored lines indicate 
the contribution rate for the positive class, 
and negative values indicate the contribution rate 
for the negative class.
Note that, because it is difficult to visualize the shapelets
over the Hilbert space associated with the Gaussian kernel,
we plotted each of them to match the original time series based on 
the Euclidean distance.
Unlike the previous visualization analyses
\citep[see, e.g.,][]{Ye:2009:TSS:1557019.1557122}, 
our visualization does not exactly interpret 
the final hypothesis because 
of the non-linear feature map.
However, we can deduce that the colored lines represent
``important patterns'', which make significant
contributions to classification.
\begin{figure}[h!]
\centering
\begin{tabular}{c}
\begin{minipage}{0.5\hsize}
\centering
  \includegraphics[width=50mm, height=45mm]{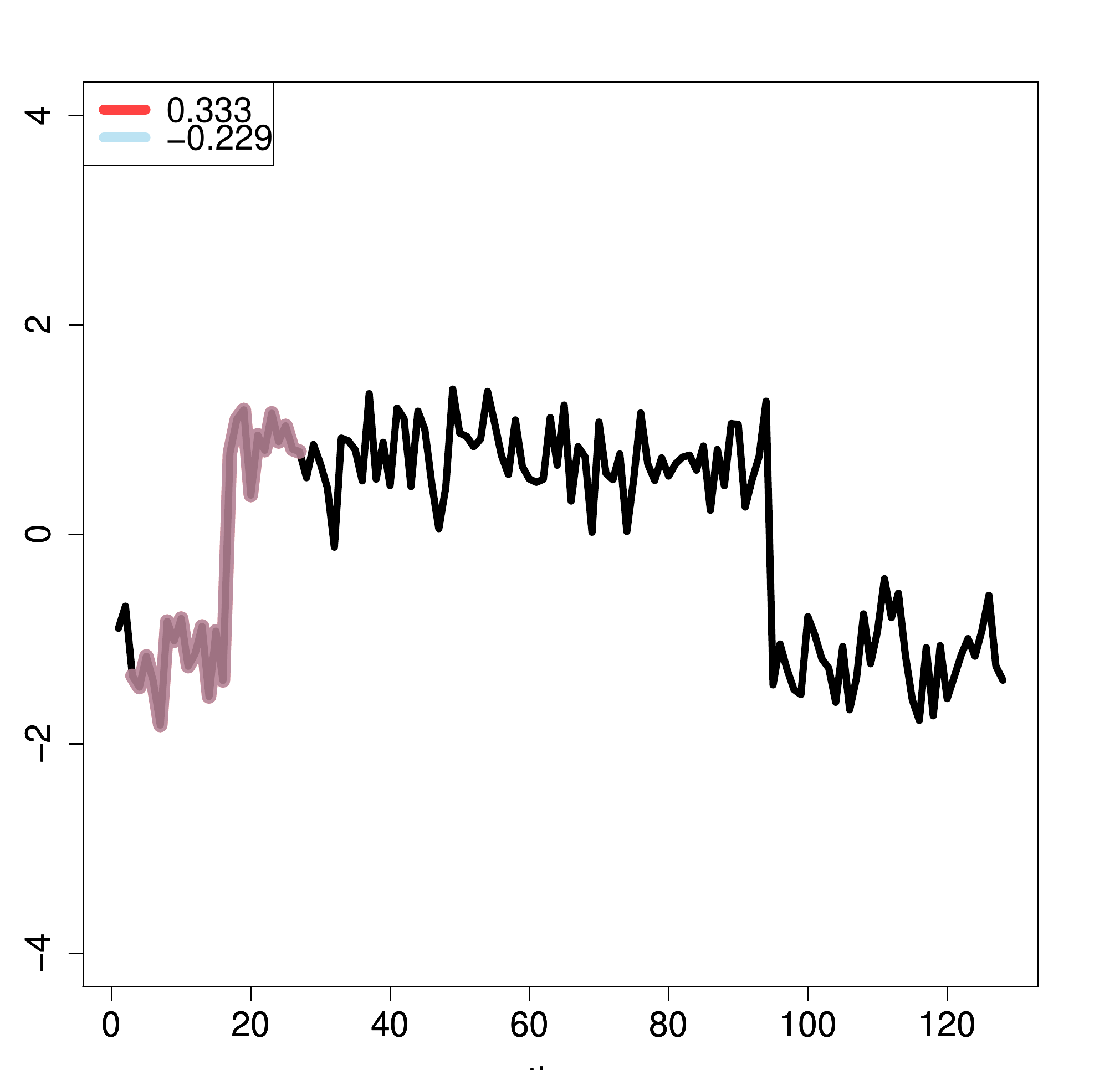}
  \\(cylinder) 
\end{minipage}
\begin{minipage}{0.5\hsize}
\centering
\includegraphics[width=50mm, height=45mm]{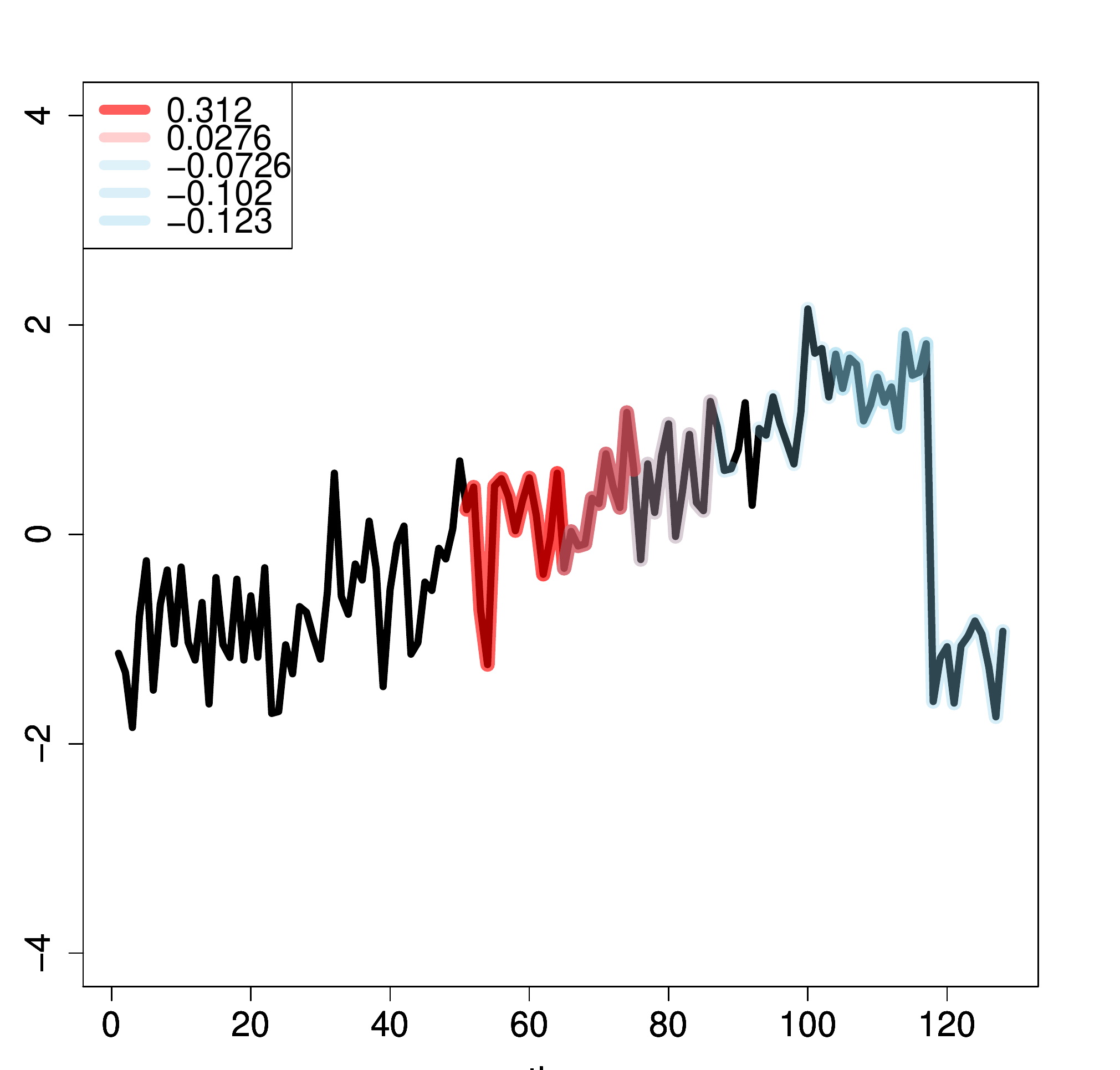}
        \hspace{1.6cm} (bell) 
\end{minipage}\\
\begin{minipage}{0.5\hsize}
\centering
\includegraphics[width=50mm, height=45mm]{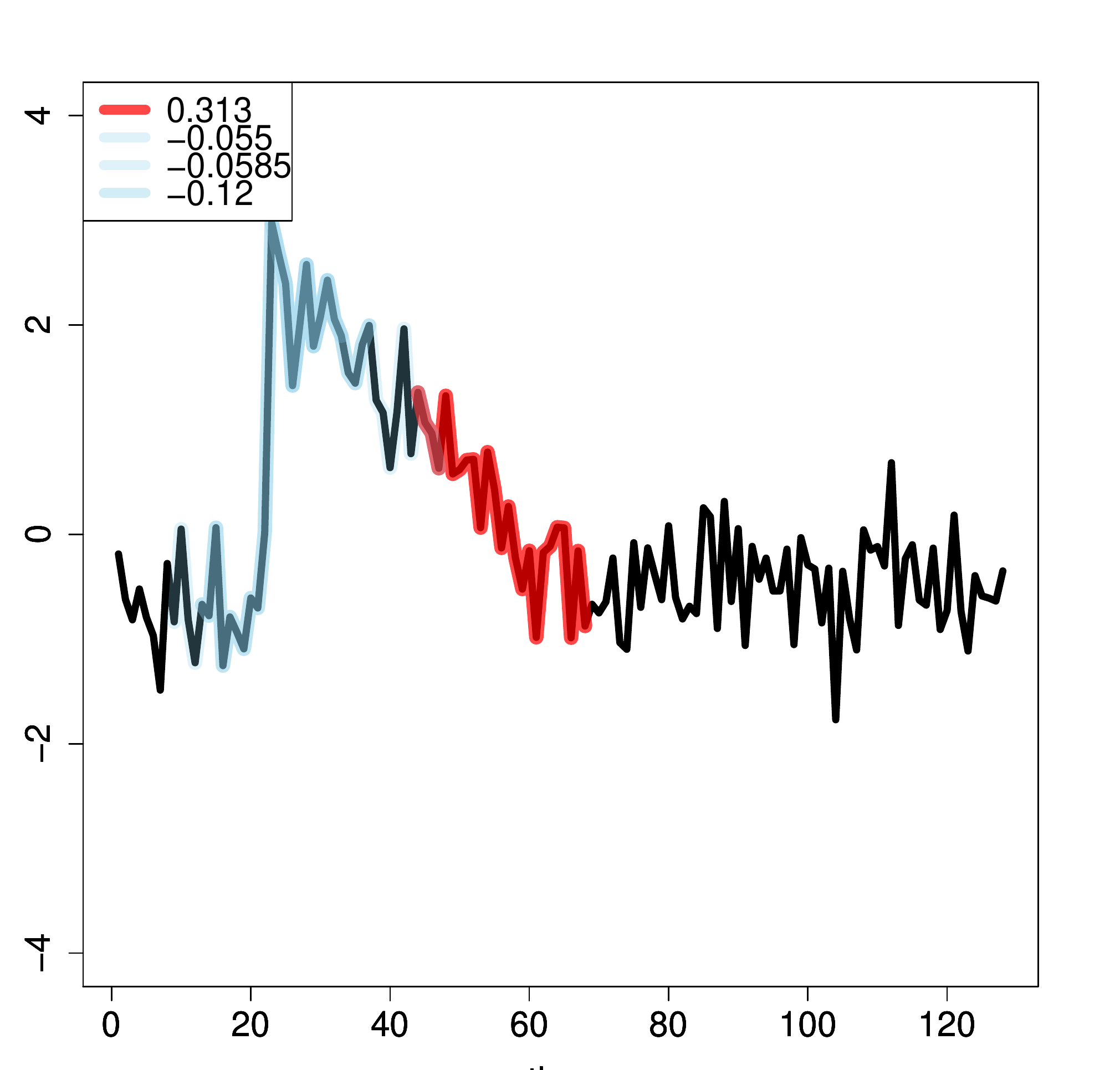}
        \hspace{1.6cm} (funnel) 
\end{minipage}
\end{tabular}
\caption{Examples of the visualization of maximizers for a CBF time-series
   data. Black lines are original time series.
   We highlight each subsequence that maximizes the
   similarity with some shapelet in a classifier. 
   Subsequences with positive values (red) contribute to the positive class and
   subsequences with negative values (blue) contribute to the negative
   class. \label{fig:cbf_subseq}}
\end{figure}

\begin{figure}[h!]
\centering
\begin{tabular}{c}
\begin{minipage}{0.5\hsize}
\centering
  \includegraphics[width=50mm, height=45mm]{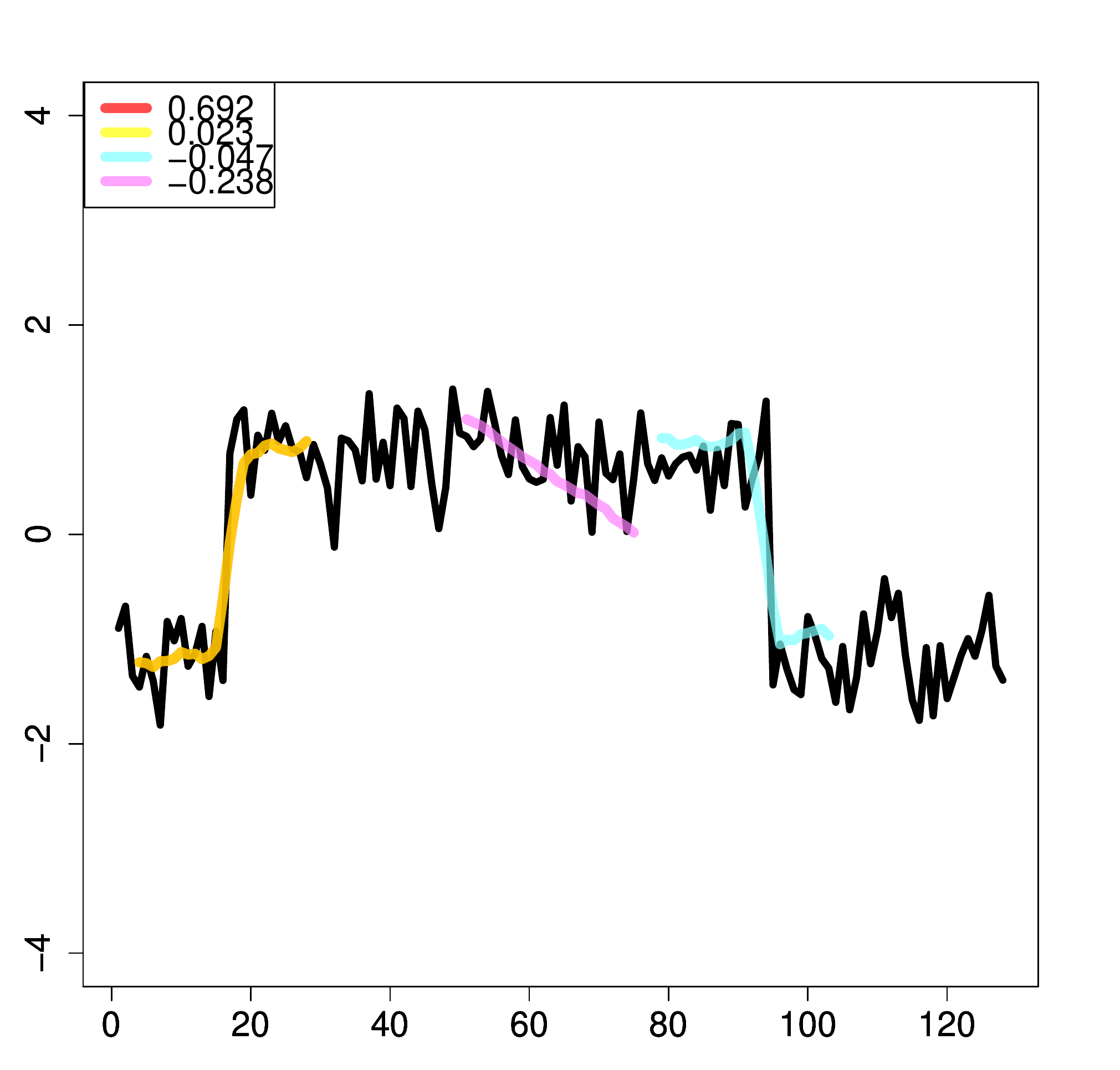}
  \\(cylinder) 
\end{minipage}
\begin{minipage}{0.5\hsize}
\centering
\includegraphics[width=50mm, height=45mm]{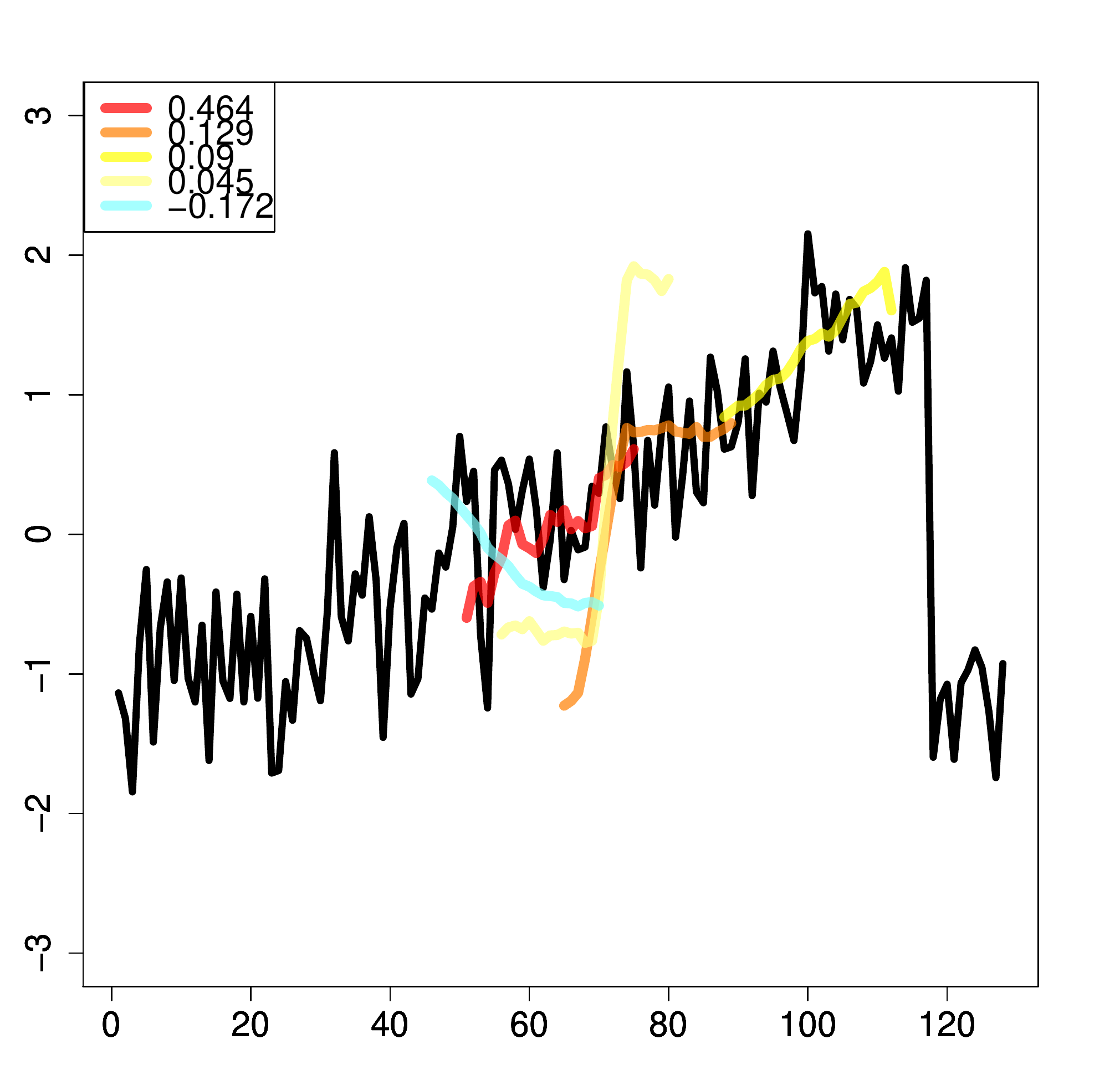}
        \hspace{1.6cm} (bell) 
\end{minipage}\\
\begin{minipage}{0.5\hsize}
\centering
\includegraphics[width=50mm, height=45mm]{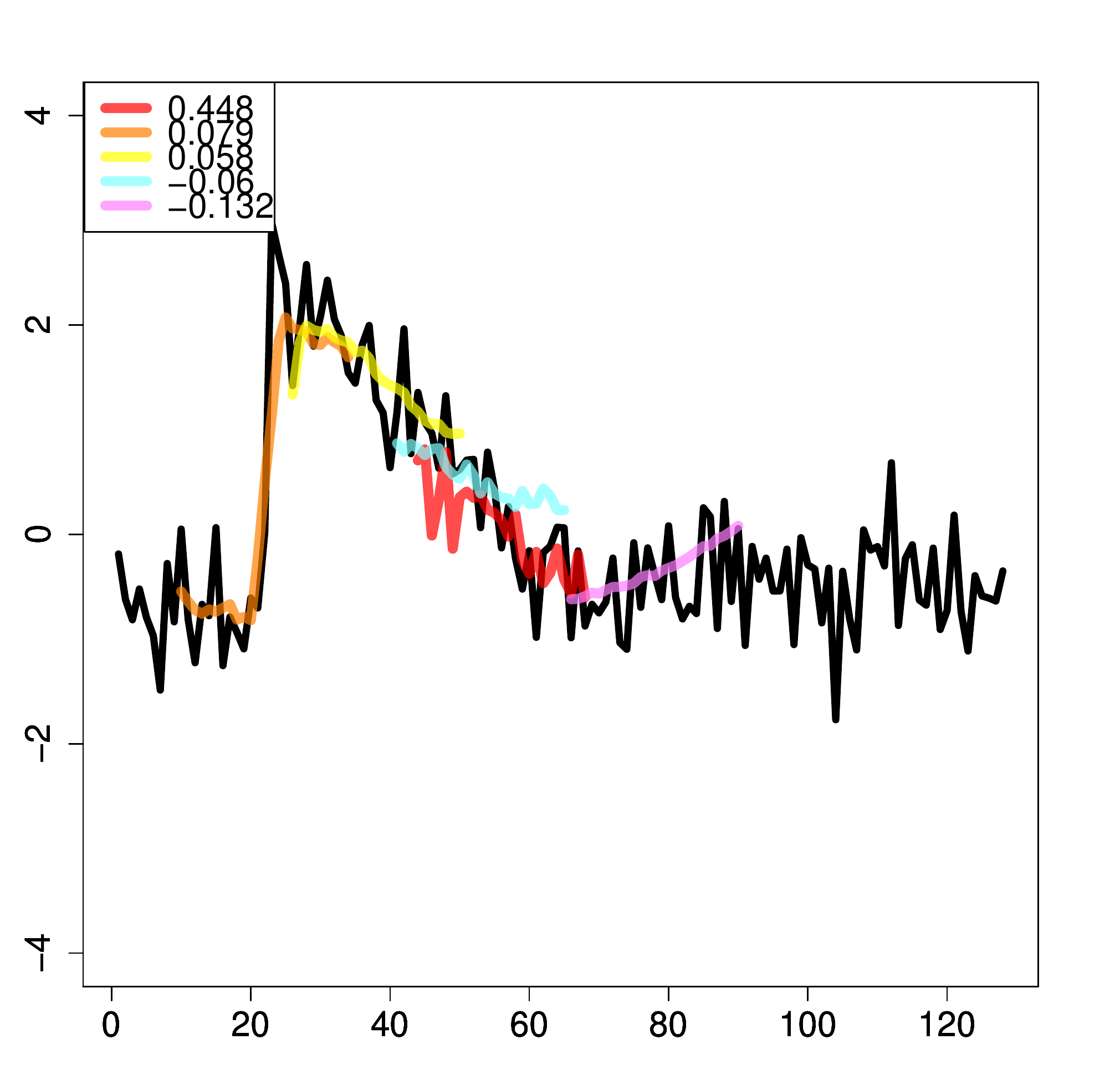}
        \hspace{1.6cm} (funnel) 
\end{minipage}
\end{tabular}
\caption{Examples of the visualization of shapelets for a CBF time-series data.
The colored lines show important patterns of the obtained classifier.
Positive values on the colored lines (red to yellow) 
 indicate the contribution rate for the positive class,
  and negative values (blue to purple) indicate the
  contribution rate for the negative class. \label{fig:cbf_shapelet}}
\end{figure}
\subsection{Results for Multiple-Instance Data}
We selected the baselines of MIL algorithms as
mi-SVM and MI-SVM~\citep{NIPS2002misvm}, and MILES~\citep{MI1normSVM}.
mi-SVM and MI-SVM are classical method in MIL, but still perform favorably
compared with state-of-the-art methods
for standard multiple-instance data \citep[see, e.g.,][]{doran:thesis}.
The details of the datasets are shown in Table~\ref{tab:mi_data}.
\par
\modiff{
mi- and MI-SVM find a single but an optimized shapelet ${\bf u}$
which is not limited to the instance in the training sample. 
The classifiers obtained by these algorithms are formulated as:
\begin{align}
\label{align:misvm}
g(B) = \max_{x \in B} \langle {\bf u}, \Phi(x) \rangle = \max_{x \in B}\sum_{z \in P_S} \alpha_{z} K (z, x).
\end{align}
MILES finds the multiple-shapelets, but they are limited to the instances in the training sample.
The classifier of MILES is formulated as follows:
\begin{align}
\label{align:miles}
g(B) = \sum_{z \in P_S} w_{z} \max_{x \in B}K(z, x).
\end{align}
}


We used the implementation provided by
Doran\footnote{https://github.com/garydoranjr/misvm}~for mi-SVM
and MI-SVM.
We combined the Gaussian kernel 
with mi-SVM and MI-SVM.
Parameter $C$ was chosen from $\{1, 10, \linebreak[0] 100, 1000,
10000\}$.\rmmodif{(remove the sentence here ``and parameter...'')}
For our method and MILES\footnote{MILES uses 1-norm SVM to obtain a
  final classifier. We implemented 1-norm SVM by using the formulation
  of~\cite{warmuth-etal:nips08}}, we chose $\nu$ from $\{0.5, \allowbreak 0.3, 0.2, 0.15, 0.1\}$,
and we only used the Gaussian kernel. 
Furthermore, we chose $\sigma$ from $\{0.005, \allowbreak 0.01, 0.05, 0.1, 0.5, 1.0 \}$.
We use $100$-means clustering with respect to each class to reduce $P_S$.
\modiff{
To avoid the randomness of $k$-means, 
we ran 30 times of training and selected the model which achieved
the best training accuracy.
}
For efficiency,
we demonstrated the weak learning problem \textsf{OP} 2.
For all these algorithms, we estimated optimal parameter set 
via 
5-fold cross-validation.
We used well-known multiple-instance data
as shown on the left-hand side of Table~\ref{tab:acc2}.
The accuracies resulted from 10 times of 5-fold cross-validation.

\begin{table*}[h!]
\centering
\caption{Details of MIL datasets. \label{tab:mi_data}}
\begin{tabular}{|c ||c|c|c| } \hline
dataset  & sample size& \# instances & dimension
 \\ \hline
MUSK1 &  $92$ & $476$ & $166$  
\\ 
MUSK2 & $102$& $6598$ & $166$ 
\\ 
elephant & $200$& $1391$ &$230$
\\ 
fox & $200$ & $1320$ & $230$
\\ 
tiger & $200$& $1220$ & $230$ 
\\ \hline
\end{tabular}
\end{table*}

\begin{table*}[h!]
\centering
\caption{Classification
  accuracies for MIL datasets. \label{tab:acc2}}
\begin{tabular}{|c ||c|c|c|c| } \hline
dataset & mi-SVM&
MI-SVM & MILES & 
Ours 
 \\ \hline
MUSK1 &  $0.834 \pm 0.084$ & 
$0.820 \pm 0.081$ & {$ \bf 0.865 \pm 0.068$}  &
$0.844 \pm 0.076$ 
\\ 
MUSK2 & $0.749 \pm 0.082$& 
$0.840 \pm 0.074$ & {$ 0.871 \pm 0.072$} &
{$\bf 0.879 \pm 0.067$} 
\\ 
elephant & $0.785 \pm 0.070$& 
$0.823 \pm 0.056$ &$0.796 \pm 0.068$&
{$\bf 0.828 \pm 0.061$}
\\ 
fox & $0.618 \pm 0.069$ & 
$0.578 \pm 0.075$ & ${\bf 0.675 \pm 0.071}$& 
{$  0.646 \pm 0.063$} 
\\ 
tiger & $0.752 \pm 0.078$& 
$0.815 \pm 0.055$ & ${\bf 0.827 \pm 0.057}$ & 
{$ 0.817 \pm 0.058$}  
\\ \hline
\end{tabular}
\end{table*}

\begin{table*}[h!]
\centering
\caption{Training
  accuracies for MIL datasets. \label{tab:acc3}}
\begin{tabular}{|c ||c|c| } \hline
dataset & MILES & 
Ours 
 \\ \hline
MUSK1 &  $0.987$ & 
$0.985$ 
\\ 
MUSK2 & $0.980$& 
 $0.993$
\\ 
elephant & $0.963$& 
$ 0.993$
\\ 
fox & $0.987$ & 
$  0.995$ 
\\ 
tiger & $0.973$& 
$ 0.993$
\\ \hline
\end{tabular}
\end{table*}


\begin{table*}[h!]
\centering
\caption{Training time (sec.) for MIL datasets. \label{tab:mi_traintime}}
\begin{tabular}{|c ||c|c|c|c| } \hline
dataset & mi-SVM&
MI-SVM & MILES & 
Ours 
 \\ \hline
MUSK1 &  $29.6$ & 
$28.1$ & $0.584$  &
$5.57$ 
\\ 
MUSK2 & $3760.1$& 
$3530.0$ & $103.1$ &
$80.5$ 
\\ 
elephant & $240.6$& 
$130.3$ &$5.84$&
$8.30$
\\ 
fox & $201.9$ & 
$139.2$ & $5.4$& 
$26.4$ 
\\ 
tiger & $158.5$& 
$118.0$ & $4.6$ & 
$9.8$  
\\ \hline
\end{tabular}
\end{table*}

\begin{table*}[h!]
\centering
\caption{Testing time (sec.) for MIL datasets. \label{tab:mi_testtime}}
\begin{tabular}{|c ||c|c|c|c| } \hline
dataset  & mi-SVM&
MI-SVM & MILES & 
Ours 
 \\ \hline
MUSK1 &  $0.010$ & 
$0.004$ & $0.011$  &
$0.045$ 
\\ 
MUSK2 & $0.577$& 
$0.063$ & $0.129$ &
$0.083$ 
\\ 
elephant & $0.053$& 
$0.015$ &$0.067$&
$0.115$
\\ 
fox & $0.078$ & 
$0.025$ & $0.118$& 
$0.145$ 
\\ 
tiger & $0.059$& 
$0.012$ & $0.065$ & 
$0.118$  
\\ \hline
\end{tabular}
\end{table*}
\modiff{
The results are shown in Table~\ref{tab:acc2}.
MILES and Ours achieve significantly better performance than
mi- and MI-SVM.
Ours achieves comparable performance to MILES.
Table~\ref{tab:acc3} shows the training accuracies of
MILES and Ours.
It can be seen that Ours achieves higher training accuracy.
This result is theoretically reasonable
because our hypothesis class is richer than MILES.
However, in other words, this means that Ours
has a higher overfitting risk than MILES.
}
\par
\modiff{
Table~\ref{tab:mi_traintime} shows that the training time of the five methods.
It is clear that MILES and Ours are more efficient than mi- and MI-SVM.
The main reason is that mi- and MI-SVM solve Quadratic Programming (QP) problem
while MILES and Ours solve LP problems.
MILES worked averagely more efficient than Ours. 
However, for MUSK2 which has a large number of instances,
Ours worked more efficiently than MILES.
}
\par
\modiff{
The testing time of each algorithm is shown in Table~\ref{tab:mi_testtime}.
We can see that Ours is
comparable to the other algorithms.
}

\section{Conclusion and Future Work}
\label{sec:conclusion}
We proposed a new MIL formulation that
provides a richer class of the final classifiers based 
on infinitely many shapelets.
We derived the tractable formulation
over infinitely many shapelets
with theoretical support, and provided an algorithm
based on LPBoost and DC (Difference of Convex) algorithm. 
Our result gives theoretical justification for some existing
shapelet-based classifiers
\citep[e.g.,][]{MI1normSVM,Hills:2014:CTS:2597434.2597448}. 
The experimental results demonstrate that 
the provided approach uniformly works for SL and MIL tasks
without introducing domain-specific parameters and heuristics, and
compares with the baselines of shapelet-based classifiers.

Especially for time-series classification, 
the number of instances usually becomes large.
Although we took 
a
heuristic approach in the experiment,
we think it is not an essential solution to improve the efficiency.
We preliminarily implemented \textsf{OP} 1 with
Orthogonal Random Features~\citep{NIPS2016_6246}
that can approximate the Gaussian kernel accurately.
It allows us to solve the primal problem of \textsf{OP} 1
directly, and allows us to avoid constructing 
a
large kernel matrix.
The implementation improved the efficiency drastically;
however, it did not achieve high accuracy 
as compared with solutions of \textsf{OP} 2 with the heuristics.
For SL tasks, there are many successful efficient methods
using some heuristics specialized in time-series domain~\citep{KeoghR13,renard:hal-01217435,GrabockaWS15,WistubaGS15,HouKZ16,Karlsson:2016}.
We will explore many ways to
improve efficiency for SL tasks.

Moreover,
we would like to 
improve the generalization error bound.
Our bound is still incomparable with the existing bound.
Since we think it requires to study more complex analysis,
we reserve this for future work.
Our heuristics might reduce the model complexity
(i.e., the risk of overfitting); 
however, we do not know how the complexity can be reduced by
our heuristics theoretically.
To apply our method to various domains, we would like to explore the
general techniques for reducing overfitting risk of our method.

\section*{Acknowledgement}
We would like to thank Prof. Eamonn Keogh and all the people who have contributed to the UCR time series classification archive.
This work is supported by JST CREST (Grant Number JPMJCR15K5)
and JSPS KAKENHI (Grant Number JP18K18001).
In the experiments, we used the computer resource offered
under the category of General Projects by Research Institute for
Information Technology, Kyushu University.
\bibliographystyle{apa}

\clearpage
\appendix
\section{Proof of Theorem~\ref{theo:represent}}
\label{sec:proof1}
First, we give a definition for convenience.
\begin{defi}
\label{def:theta}
{\rm [The set $\Theta$ of mappings from a bag to an instance]} \\ 
Given a sample $S=(B_1, \dots, B_m)$.
For any ${\bf u} \in U$, let $\theta_{{\bf u}, \Phi}: \{B_1, \ldots, B_m\} \to {\mathcal{X}}$ be a
mapping defined by
\[
\theta_{{\bf u}, \Phi}(B_i) := \arg\max_{x \in B_i} 
 \left\langle {\bf u},  \Phi\left(x\right)  \right\rangle, 
\]
and we define the set of all $\theta_{{\bf u}, \Phi}$ for $S$ as 
$\Theta_{S, \Phi} = \{\theta_{{\bf u}, \Phi} \mid {\bf u} \in U \}$.
For the sake of brevity, $\theta_{{\bf u}, \Phi}$ and $\Theta_{S, \Phi}$
will be abbreviated as $\theta_{{\bf u}}$ and $\Theta$, respectively.
\end{defi}
Below we give a proof of Theorem~\ref{theo:represent}.
 \begin{proof}
We can rewrite the optimization problem (\ref{align:WeakLearn_u})
by using $\theta \in \Theta$ as follows:
\begin{align}
 \max_{\theta \in \Theta}\max_{{\bf u} \in {\mathbb{H}}: \theta_{\bf u}=\theta} \quad& 
\sum_{i=1}^my_id_i
\left\langle {\bf u},  \Phi\left(\theta(B_i)\right)  \right\rangle \\
 \nonumber
\text{sub.to} \quad& \|{\bf u}\|_{\mathbb{H}}^2 \leq 1.
\end{align}
Thus, if we fix $\theta \in \Theta$, we have a sub-problem. Since
  the constraint $\theta=\theta_{\bf u}$ can be written as 
 the number $|P_S|$ of linear constraints 
(i.e., sub.to $\langle {\bf u}, \Phi(x) \rangle \leq \langle {\bf u},
\Phi(\theta(B_i)) \rangle \; (i \in [m], x \in B_i)$), 
each sub-problem is equivalent to a convex optimization.  
Indeed, each sub-problem can be written as the equivalent unconstrained
  minimization (by neglecting constants in the objective)
\begin{align*}
\min_{{\bf u} \in {\mathbb{H}}} ~&\beta \|{\bf u}\|^2_{\mathbb{H}} -
\sum_{i=1}^m\sum_{x \in B_i} \eta_{i, x} \left(
 \left\langle {\bf u}, \Phi\left(\theta(B_i)\right) \right\rangle -  
\left\langle {\bf u}, \Phi(x)  \right\rangle
\right) 
  -\sum_{i=1}^my_id_i
 \left\langle {\bf u},  \Phi\left(\theta(B_i)\right)  \right\rangle
\end{align*}
where $\beta$ and $\eta_{i,x}$ $(i \in [m], x \in B_i)$ are 
the corresponding positive constants.
Now for each sub-problem, we can apply the standard Representer Theorem
  argument~\citep[see, e.g., ][]{Mohri.et.al_FML}).
  Let ${\mathbb{H}}_1$ be the subspace $\{{\bf u} \in {\mathbb{H}} \mid
 {\bf u}=\sum_{z \in P_S} \alpha_{z}\Phi(z), \alpha_{z}\in \mathbb{R}\}$.
 We denote ${\bf u}_1$ as the orthogonal projection of ${\bf u}$ onto
 ${\mathbb{H}}_1$ and any ${\bf u}\in {\mathbb{H}}$ has the decomposition
 ${\bf u}={\bf u}_1 + {\bf u}^{\perp}$. Since ${\bf u}^{\perp}$ is orthogonal
 w.r.t. ${\mathbb{H}}_1$, $\|{\bf u}\|_{\mathbb{H}}^2=\|{\bf u}_1\|_{\mathbb{H}}^2 +
 \|{\bf u}^{\perp}\|_{\mathbb{H}}^2 \geq \|{\bf u}_1\|_{\mathbb{H}}^2$.
On the other hand, $\left\langle {\bf u},  \Phi\left(z \right)
 \right\rangle =\left\langle {\bf u}_1,  \Phi\left(z \right)  \right\rangle$.
 Therefore, the optimal solution of each sub-problem has to be contained
  in ${\mathbb{H}}_1$.
  This implies that the optimal solution, which is the maximum over all
  solutions of sub-problems, is contained in ${\mathbb{H}}_1$ as well.
 \end{proof}

\section{Proof of Theorem~\ref{theo:main}}
\label{proof2}
We use $\theta$ and $\Theta$ of Definition~\ref{def:theta}.

\begin{defi}
{\rm [The Rademacher and the Gaussian complexity~\citep{Bartlett:2003:RGC}]}\\
Given a sample $S=(x_1,\dots,x_m) \in {\mathcal{X}}^m$, 
the empirical Rademacher complexity ${\mathfrak{R}}(H)$ of a class $H \subset
 \{h: {\mathcal{X}} \to \mathbb{R}\}$ w.r.t.~$S$
 is defined as 
 ${\mathfrak{R}}_S(H)=\frac{1}{m}\mathop{\rm  E}\limits_{{\boldsymbol{\sigma}}}\left[
\sup_{h \in H}\sum_{i=1}^m \sigma_i h(x_i)
 \right]$,
 where ${\boldsymbol{\sigma}} \in \{-1,1\}^m$ and each $\sigma_i$ is an independent
 uniform random variable in $\{-1,1\}$.
 The empirical Gaussian complexity ${\mathfrak{G}}_S(H)$ of $H$ w.r.t.~$S$
 is defined similarly but
 each $\sigma_i$ is drawn independently from the standard normal distribution.
\end{defi} 

The following bounds are well-known. 
\begin{lemm}
 \label{lemm:RC_and_GC}
{\rm [Lemma 4 of~\citep{Bartlett:2003:RGC}]}
 ${\mathfrak{R}}_S(H) =O({\mathfrak{G}}_S(H))$.
\end{lemm}

\begin{lemm}
\label{lemm:ensemble_margin_bound}
{\rm [Corollary 6.1 of~\citep{Mohri.et.al_FML}]} 
For fixed $\rho$, $\delta >0$, 
the following bound holds with probability at least $1- \delta$:
for all $f \in {\mathrm{conv}(H)}$,
\[
{\mathcal{E}}_D(f) \leq {\mathcal{E}}_{\rho}(f) + \frac{2}{\rho} {\mathfrak{R}}_S(H) + 3 \sqrt{\frac{\log\frac{1}{\delta}}{2m}}.
\]
\end{lemm}

To derive generalization bound based on the Rademacher or the Gaussian
complexity is quite standard in the statistical learning theory
literature and applicable to our classes of interest as well.
However, a standard analysis provides us sub-optimal bounds. 

\begin{lemm}
\label{lemm:GC}
Suppose that for any $z \in {\mathcal{X}}$, $\|\Phi(z)\|_{\mathbb{H}} \leq R$.
Then, the empirical Gaussian complexity of $H_U$ with respect to $S$
for $U \subseteq \{{\bf u} \mid \|{\bf u}\|_{\mathbb{H}} \leq 1\}$
is bounded as follows: 
\[
{\mathfrak{G}}_{S}(H) \leq \frac{R\sqrt{(\sqrt{2}-1)+ 2(\ln|\Theta|)}}{\sqrt{m}}.
\]
\end{lemm}
\begin{proof} 
Since $U$ can be partitioned into
	$\bigcup_{\theta \in \Theta} \{ {\bf u} \in U \mid \theta_{\bf u} = \theta\}$,
\begin{align}
\label{eq:gauss1}
\nonumber
{\mathfrak{G}}_{S}(H_U) &= \frac{1}{m} \mathop{\rm  E}\limits_{{\boldsymbol{\sigma}}} \left[\sup_{\theta \in \Theta} 
\sup_{{\bf u} \in U:\theta_{{\bf u}}=\theta} \sum_{i=1}^m \sigma_i
\left\langle {\bf u}, \Phi\left(\theta(B_i)\right)
             \right\rangle\right] \\ \nonumber
&= \frac{1}{m} \mathop{\rm  E}\limits_{{\boldsymbol{\sigma}}} \left[\sup_{\theta \in \Theta} 
\sup_{{\bf u} \in U:\theta_{{\bf u}}=\theta} \left\langle {\bf u},
 \left(\sum_{i=1}^m \sigma_i \Phi\left(\theta(B_i)\right)
  \right) \right \rangle \right] \\ \nonumber
&\leq \frac{1}{m} \mathop{\rm  E}\limits_{{\boldsymbol{\sigma}}} \left[\sup_{\theta \in \Theta} 
\sup_{{\bf u} \in U} \left\langle {\bf u}, 
\left(\sum_{i=1}^m \sigma_i \Phi\left(\theta(B_i)\right)
  \right) \right \rangle \right] \\ \nonumber
  &\leq \frac{1}{m} \mathop{\rm  E}\limits_{{\boldsymbol{\sigma}}} \left[
  \sup_{\theta \in \Theta} \left\| \sum_{i=1}^m \sigma_i \Phi \left(
  \theta(B_i)\right) \right\|_{\mathbb{H}}
  \right]\\ \nonumber
  &= \frac{1}{m} \mathop{\rm  E}\limits_{{\boldsymbol{\sigma}}} \left[
  \sup_{\theta \in \Theta} \sqrt{\left\|\sum_{i=1}^m \sigma_i
  \Phi\left(\theta(B_i)\right)\right\|_{\mathbb{H}}^2}
  \right] \\ \nonumber
  &= \frac{1}{m} \mathop{\rm  E}\limits_{{\boldsymbol{\sigma}}} \left[
  \sqrt{\sup_{\theta \in \Theta} \left\|\sum_{i=1}^m \sigma_i
  \Phi\left(\theta(B_i)\right)\right\|_{\mathbb{H}}^2}
  \right]\\ 
& \leq \frac{1}{m} \sqrt{ 
  \mathop{\rm  E}\limits_{{\boldsymbol{\sigma}}} \left[
  \sup_{\theta \in \Theta} \left\|\sum_{i=1}^m \sigma_i
  \Phi\left(\theta(B_i)\right)\right\|_{\mathbb{H}}^2
  \right]}.
\end{align}
The first inequality is derived from the relaxation of ${\bf u}$,
the second inequality is due to Cauchy-Schwarz inequality and
the fact $\|{\bf u}\|_{{\mathbb{H}}} \leq 1$, and
the last inequality is due to Jensen's inequality.
We denote by ${\boldsymbol{\mathrm{K}}}^{(\theta)}$ the kernel matrix such that 
${\boldsymbol{\mathrm{K}}}_{ij}^{(\theta)} = \langle \Phi((\theta(B_i)),
\Phi(\theta(B_j)) \rangle$.
 Then, we have
 \begin{align}
  \mathop{\rm  E}\limits_{{\boldsymbol{\sigma}}} \left[
  \sup_{\theta \in \Theta} \left\|\sum_{i=1}^m \sigma_i
  \Phi\left(\theta(B_i)\right)\right\|_{\mathbb{H}}^2
  \right]
  =
  \mathop{\rm  E}\limits_{{\boldsymbol{\sigma}}}\left[ \sup_{\theta \in \Theta}  
\sum_{i,j=1}^m \sigma_i\sigma_j 
{\boldsymbol{\mathrm{K}}}_{ij}^{(\theta)}
\right].
 \end{align}
We now derive an upper bound of the r.h.s. as follows.
 
For any $c>0$, 
\begin{align*}
&\exp\left(
c \mathop{\rm  E}\limits_{{\boldsymbol{\sigma}}}\left[ \sup_{\theta \in \Theta}  
\sum_{i,j=1}^m \sigma_i\sigma_j 
{\boldsymbol{\mathrm{K}}}_{ij}^{(\theta)}
\right]
\right) \\
&\leq 
\mathop{\rm  E}\limits_{{\boldsymbol{\sigma}}}\left[\exp\left(
c \sup_{\theta \in \Theta}  
\sum_{i,j=1}^m \sigma_i\sigma_j 
{\boldsymbol{\mathrm{K}}}_{ij}^{(\theta)}
\right)
\right] \\
= &
\mathop{\rm  E}\limits_{{\boldsymbol{\sigma}}}\left[
\sup_{\theta \in \Theta}  
\exp\left(c
\sum_{i,j=1}^m \sigma_i\sigma_j 
{\boldsymbol{\mathrm{K}}}_{ij}^{(\theta)}
\right)
\right] \\
&\leq 
\sum_{\theta \in \Theta}
\mathop{\rm  E}\limits_{{\boldsymbol{\sigma}}}\left[
\exp\left(c
\sum_{i,j=1}^m \sigma_i\sigma_j 
{\boldsymbol{\mathrm{K}}}_{ij}^{(\theta)}
\right)
\right]
\end{align*}
The first inequality is due to Jensen's inequality, and
the second inequality is due to the fact that the supremum is
bounded by the sum.
By using the symmetry property of
${\boldsymbol{\mathrm{K}}}^{(\theta)}$, we have
$\sum_{i,j=1}^m\sigma_i\sigma_j{\boldsymbol{\mathrm{K}}}_{ij}^{(\theta)}={\boldsymbol{\sigma}}^\top{\boldsymbol{\mathrm{K}}}^{(\theta)}{\boldsymbol{\sigma}}$,
which is rewritten as
\[
\top{{\boldsymbol{\sigma}}}{\boldsymbol{\mathrm{K}}}^{(\theta)}{\boldsymbol{\sigma}} = \top{(\top{{\boldsymbol{\mathrm{V}}}}{\boldsymbol{\sigma}})}
\left( \begin{array}{ccc}
\lambda_1^{(\theta)} & \hfill & 0  \\
\hfill & \ddots & \hfill  \\
0 & \hfill & \lambda_m^{(\theta)} \\
\end{array} \right)
\top{{\boldsymbol{\mathrm{V}}}}{\boldsymbol{\sigma}},
\]
where
$\lambda_1^{(\theta)}\geq \dots \geq \lambda_m^{(\theta)}\geq 0$ are the
eigenvalues of ${\boldsymbol{\mathrm{K}}}^{(\theta)}$ and 
${\boldsymbol{\mathrm{V}}} = ({{\bf v}}_1, \ldots, {{\bf v}}_m)$
is the orthonormal matrix such that ${{\bf v}}_i$ is the eigenvector
that corresponds to the eigenvalue $\lambda_i$.
By the reproductive property of Gaussian distribution,
$\top{{\boldsymbol{\mathrm{V}}}} {\boldsymbol{\sigma}}$ obeys the same Gaussian distribution as well.
So,
\begin{align*}
&\sum_{\theta \in \Theta}
\mathop{\rm  E}\limits_{{\boldsymbol{\sigma}}}\left[
\exp\left(c
\sum_{i,j=1}^m \sigma_i\sigma_j 
{\boldsymbol{\mathrm{K}}}_{ij}^{(\theta)}
\right)
\right] \\
=&
\sum_{\theta \in \Theta}
\mathop{\rm  E}\limits_{{\boldsymbol{\sigma}}}\left[
\exp\left(c
\top{{\boldsymbol{\sigma}}}{\boldsymbol{\mathrm{K}}}^{(\theta)}{\boldsymbol{\sigma}}
\right)
\right] \\
= &
\sum_{\theta \in \Theta}
\mathop{\rm  E}\limits_{{\boldsymbol{\sigma}}}\left[
\exp\left(c
\sum_{k=1}^m\lambda_k^{(\theta)}(\top{{{\bf v}}_k}{\boldsymbol{\sigma}})^2
\right)
\right] \\
= &
\sum_{\theta \in \Theta}
\Pi_{k=1}^m
\mathop{\rm  E}\limits_{\sigma_k}\left[
\exp\left(c
\lambda_k^{(\theta)}\sigma_k^2
\right)
\right] ~~(\text{replace}~{\boldsymbol{\sigma}} = \top{{{\bf v}}_k}{\boldsymbol{\sigma}})\\
 = &
 \sum_{\theta \in \Theta}
\Pi_{k=1}^m
 \left(
 \int_{-\infty}^{\infty}
 \exp\left(
 c \lambda_k^{(\theta)}\sigma^2
 \right)
  \frac{\exp(-\sigma^2)}{\sqrt{2\pi}}d\sigma
\right) \\
 = &
\sum_{\theta \in \Theta}
\Pi_{k=1}^m
 \left(
  \int_{-\infty}^{\infty}
  \frac{\exp(-(1-c\lambda_k^{(\theta)})\sigma^2)}{\sqrt{2\pi}}d\sigma\right).
\end{align*}
Now we replace $\sigma$ by $\sigma' =
\sqrt{1-c\lambda_k^{(\theta)}}\sigma$. Since
$d\sigma'=\sqrt{1-c\lambda_k^{(\theta)}}d\sigma$,
we have:
\begin{align*}
&\int_{-\infty}^{\infty}
\frac{\exp(-(1-c\lambda_k^{(\theta)})\sigma^2)}{\sqrt{2}\pi}d\sigma 
= \frac{1}{\sqrt{2\pi}} \int_{-\infty}^{\infty} 
\frac{\exp(-\sigma'^2)}{\sqrt{1-c\lambda_k^{(\theta)}}}d\sigma' \\
= &\frac{1}{\sqrt{1-c\lambda_k^{(\theta)}}}.
\end{align*}
 Now,  applying the inequality that $\frac{1}{\sqrt{1-x}} \leq
 1+2(\sqrt{2}-1)x$ for  $0 \leq x \leq \frac{1}{2}$, 
 the bound becomes
 \begin{align}
\label{align:ineq1}
\nonumber
 &\exp\left(
c \mathop{\rm  E}\limits_{{\boldsymbol{\sigma}}}\left[ \sup_{\theta \in \Theta}  
\sum_{i,j=1}^m \sigma_i\sigma_j 
{\boldsymbol{\mathrm{K}}}_{ij}^{(\theta)}
\right]
  \right) \\
\leq
&\sum_{\theta \in \Theta}
\Pi_{k=1}^m
  \left(
1+2(\sqrt{2}-1)c\lambda_k^{(\theta)} + 2\lambda_1
  \right).
 \end{align}
Further, taking logarithm, dividing the both sides by $c$,
letting $c=\frac{1}{2
  \max_{k}\lambda_k^{(\theta)}}=1/(2\lambda_1^{(\theta)})$,
fix $\theta = \theta^*$ such that $\theta^*$ maximizes (\ref{align:ineq1}),
 and applying
 $\ln(1+x) \leq x$, we get:
\begin{align}
\label{eq:gauss2}
\nonumber
&\mathop{\rm  E}\limits_{{\boldsymbol{\sigma}}}\left[ \sup_{\theta \in \Theta}  
\sum_{i,j=1}^m \sigma_i\sigma_j 
{\boldsymbol{\mathrm{K}}}_{ij}^{(\theta^*)}
\right]
\\ \nonumber
&\leq
 (\sqrt{2}-1)\sum_{k=1}^m \lambda_k^{(\theta^*)}  + 2\lambda_1^{(\theta^*)} \ln |\Theta|\\
\nonumber
 &=
 (\sqrt{2}-1)\top{({\boldsymbol{\mathrm{K}}}^{(\theta^*)})} + 2\lambda_1^{(\theta^*)} \ln |\Theta|\\
 &\leq
 (\sqrt{2}-1)m R^2 + 2 mR^2 \ln |\Theta|,
\end{align}
 where the last inequality holds since $\lambda_1^{(\theta^*)}=\|{\boldsymbol{\mathrm{K}}}^{(\theta^*)}\|_2 \leq m
 \|{\boldsymbol{\mathrm{K}}}^{(\theta)}\|_{\max} \leq R^2$.
By Equation~(\ref{eq:gauss1}) and (\ref{eq:gauss2}), 
we have:
\begin{align*}
{\mathfrak{G}}_S(H) 
&\leq
\frac{1}{m}
\sqrt{
\mathop{\rm  E}\limits_{{\boldsymbol{\sigma}}}
\left[\sup_{\theta \in \Theta}  
\sum_{i,j=1}^m \sigma_i\sigma_j 
{\boldsymbol{\mathrm{K}}}_{ij}^{(\theta)}
\right]} \\
&\leq 
\frac{R \sqrt{(\sqrt{2}-1) + 2\ln 
|\Theta|
}}{\sqrt{m}}.
\end{align*}
\end{proof}

Thus, it suffices to bound the size $|\Theta|$. 
The basic idea to get our bound is the following geometric analysis.
Fix any $i \in [m]$ and consider points $\{\Phi(x) \mid x \in B_i\}$.
Then, we define equivalence classes of ${\bf u}$ such that
$\theta_{\bf u}(i)$ is in the same class, which define a Voronoi diagram
for the points $\{\Phi(x) \mid x \in B_i\}$.
Note here that the similarity is measured by the inner product, not a
distance. More precisely, 
let $\{V_i(x) \mid x \in B_i \}$ be the Voronoi diagram,
each of the region is defined as 
$V_i(x) = \{{\bf u} \in {\mathbb{H}} \mid \theta_{{\bf u}}(B_i)=x\}$
Let us consider the set of intersections
$\bigcap_{i \in [m]}V_i{(x_i)}$ for all combinations of
$(x_1, \ldots, x_m) \in B_1 \times \cdots \times B_m $.
The key observation is that each non-empty intersection corresponds to a
mapping $\theta_{{\bf u}} \in \Theta$. Thus,
we obtain $|\Theta|=(\text{the number of intersections $\bigcap_{i\in
[m]}V_{i}(x_i)$})$. In other words, the size of $\Theta$ is exactly
the number of rooms defined by the intersections of $m$ Voronoi
 diagrams $V_1,\dots,V_m$.
 From now on, we will derive upper bound based on this observation.

 \begin{lemm}
  \label{lemm:Theta}
 \[
  |\Theta| =O(|P_S|^{2d_{\Phi,S}^*}).
 \]
 \end{lemm}
\begin{proof}
 We will reduce the problem of counting intersections of the Voronoi
 diagrams to that of counting possible labelings by hyperplanes for
 some set.
 Note that for each neighboring Voronoi regions, the border is a part of
 hyperplane since the closeness is defined in terms of the inner
 product.
 Therefore, by simply extending each border to a hyperplane, we obtain
 intersections of halfspaces defined by the extended hyperplanes.
 Note that, the size of these intersections gives an upper bound of
 intersections of the Voronoi diagrams.
 More precisely, we draw hyperplanes for each pair of points in
 $\Phi(P_S)$
 so that each point on
 the hyperplane has the same inner product between two points.
 Note that for each pair $\Phi(z), \Phi(z') \in P_S$, the normal vector of
 the hyperplane is given as $\Phi(z) -\Phi(z')$ (by fixing the sign arbitrary).
 Thus, the set of hyperplanes obtained by this procedure is exactly
 ${\Phi_{\mathrm{diff}}(P_S)}$. The size of ${\Phi_{\mathrm{diff}}(P_S)}$ is ${|P_S|} \choose 2$, which is at most $|P_S|^2$.
 Now, we consider a ``dual'' space by viewing each hyperplane as a point and each point in $U$ as a
 hyperplane.
 Note that points ${\bf u}$ (hyperplanes in the dual) in an intersection 
 give the same labeling on the points in the dual domain.
 Therefore, the number of intersections in the original domain is the
 same as the number of the possible labelings on ${\Phi_{\mathrm{diff}}(P_S)}$ by hyperplanes
 in $U$. By the classical Sauer's
 Lemma and the VC dimension of hyperplanes~\citep[see, e.g., Theorem 5.5
 in][]{LK02}), the size is at most $O((|P_S|^2)^{d_{\Phi,S}^*})$.
\end{proof}


  \begin{theo}\noindent \mbox\\
   \label{theo:Theta} 
  \begin{enumerate}
   \item[(i)] For any $\Phi$, $|\Theta|=O(|P_S|^{8(R/\mu^*)^2})$.
   \item[(ii)] if ${\mathcal{X}} \subseteq \mathbb{R}^\ell$ and $\Phi$ is the identity mapping over $P_S$, then
	      $|\Theta|=O( |P_S|^{\min\{8(R/\mu^*)^2, 2\ell \}}\})$.
   \item[(iii)] if ${\mathcal{X}} \subseteq \mathbb{R}^\ell$ and $\Phi$ satisfies that $\left\langle\Phi(z), \Phi(x) \right\rangle$ is monotone
decreasing with respect to $\|z-x\|_2$  (e.g.,
	the mapping defined by the Gaussian kernel) and
 $U=\{\Phi(z) \mid z \in {\mathcal{X}} \subseteq \mathbb{R}^\ell, \|\Phi(z)\|_{\mathbb{H}} \leq 1\}$, 
then $|\Theta|=O( |P_S|^{\min\{8(R/\mu^*)^2, 2\ell \}}\})$.
  \end{enumerate}
  \end{theo}

\begin{proof}
 (i)
 We follow the argument in Lemma 4.
 For the set of classifiers $F=\{f:{\Phi_{\mathrm{diff}}(P_S)} \to \{-1,1\} \mid
 f={\mathrm{sign}}(\langle {\bf u}, {{\bf v}}\rangle), \|{\bf u}\|_{\mathbb{H}} \leq 1, 
 \min_{{{\bf v}} \in {\Phi_{\mathrm{diff}}(P_S)}}| \langle {\bf u}, {{\bf v}} \rangle|=\mu \}$,
 its VC dimension is known to be at most $R^2/\mu^2$
 for ${\Phi_{\mathrm{diff}}(P_S)} \subseteq \{{{\bf v}} \mid \|{{\bf v}}\|_{\mathbb{H}} \leq
 2R\}$~\citep[see, e.g., ][]{LK02}).
 By the definition of $\mu^*$,
 for each intersection given by hyperplanes, there always exists a point ${\bf u}$
 whose inner product between each hyperplane is at least $\mu^*$.
 Therefore, the size of the intersections is bounded by the number of possible
 labelings in the dual space by $U''=\{{\bf u} \in {\mathbb{H}} \mid
 \|{\bf u}\|_{\mathbb{H}} \leq 1,  \min_{{{\bf v}} \in {\Phi_{\mathrm{diff}}(P_S)}}| \langle {\bf u},
 {{\bf v}} \rangle|= \mu^* \}$.
 Thus, we obtain that $d_{\Phi, S}^*$ is at most $8(R / \mu^*)^2$ and by
 Lemma 4, we complete the proof of case (i).

 (ii) In this case, the Hilbert space ${\mathbb{H}}$ is contained in
 $\mathbb{R}^\ell$. Then, by the fact that VC dimension $d_{\Phi,S}^*$ is at
 most $\ell$ and Lemma 4, the statement holds.

 (iii)
If $\left\langle\Phi(z), \Phi(x) \right\rangle$ is monotone
decreasing for $\|z-x\|$,
then the following holds:
\[
\arg\max_{x \in {\mathcal{X}}} \left\langle\Phi(z), \Phi(x) \right\rangle = \arg\min_{x \in {\mathcal{X}}} \|z-x\|_2.
\]
Therefore, $\max_{{\bf u}:\|{\bf u}\|_{\mathbb{H}} = 1}\langle {\bf u}, \Phi(x)
\rangle = \|\Phi(x)\|_{\mathbb{H}}$, where ${\bf u} =
\frac{\Phi(x)}{\|\Phi(x)\|_{\mathbb{H}}}$. 
It indicates that the number of Voronoi cells
made by $V(x)=\{z \in \mathbb{R}^\ell \mid z=\arg\max_{x \in B}
  (z \cdot x) \}$ corresponds to the
$\hat{V}(x)=\{\Phi(z) \in {\mathbb{H}} \mid z=\arg\max_{x \in B}
 \langle \Phi(z), \Phi(x) \rangle \}$.
 Then, by following the same argument for the linear kernel case, we get
 the same statement.
\end{proof}
Now we are ready to prove Theorem~\ref{theo:main}.
\begin{proof}[Proof of Theorem~\ref{theo:main}]
By using Lemma~\ref{lemm:RC_and_GC}, and
 \ref{lemm:ensemble_margin_bound},
we obtain the generalization bound in terms of the Gaussian
 complexity of $H$. Then, by applying Lemma~\ref{lemm:GC} and Theorem~\ref{theo:Theta},
we complete the proof.
\end{proof}

\subsection{Hyper-parameter tuning for time-series classification}
\label{sec:exp_app2}.
In the experiment for time series classification, 
we roughly tuned $\nu$ and $\sigma^2$ of the Gaussian kernel.
As we mentioned before, we need high computation time when learning 
very large time series.
The main computational cost is to iteratively 
solve weak learning problems by using an LP (or QP) solver.
The number of constraints of the optimization problem (\ref{align:subprob})
depends on the total number of instances in negative bags.
Therefore, in the hyper-parameter tuning phase,
we finish solving each weak learning problem
by obtaining the solution of the optimization problem
(\ref{align:shape_opt}).
Using the rough weak learning problem, 
we tuned $\nu$ and $\sigma$ through a grid search
via three runs of $3$-fold cross validation.

\end{document}